\documentclass[11pt]{article} 

\usepackage{hyperref}
\usepackage{url}
\usepackage{mathtools}
\usepackage{fullpage}
\usepackage[numbers]{natbib}

\usepackage{dsfont}[sans]

\usepackage{amsmath,amssymb,amsthm, amsfonts}
\usepackage{mathbbol}
\usepackage{mathrsfs}  
\usepackage{multicol}
\usepackage{rotating}
\usepackage{enumitem, comment, xifthen}

\usepackage{mathbbol}

\usepackage{booktabs}

\usepackage[utf8]{inputenc} 
\usepackage[T1]{fontenc}    
\usepackage{hyperref}       
\usepackage{url}            
\usepackage{booktabs}       
\usepackage{amsfonts}       
\usepackage{nicefrac}       

\usepackage{macros}

\newcommand{\Term}{\ensuremath{T}}

\newcommand{\Exs}{\ensuremath{\mathbb{E}}}
\newcommand{\mprob}{\ensuremath{\mathbb{P}}}
\newcommand{\BallRad}{\ensuremath{B}}

\newcommand{\Event}{\ensuremath{\mathcal{E}}}

\newcommand{\real}{\ensuremath{\mathbb{R}}}
\newcommand{\what}{\ensuremath{\widehat{w}}}
\newcommand{\wstar}{\ensuremath{w^\star}}

\newcommand{\LeastSq}{\ensuremath{\mathcal{L}}}

\newcommand{\numobs}{\ensuremath{n}}
\newcommand{\order}{\ensuremath{\mathcal{O}}}
\newcommand{\ordertil}{\ensuremath{\widetilde{\order}}}

\newcommand{\smlinprod}[2]{\ensuremath{\big\langle #1 , \, #2 \big \rangle}}

\usepackage{natbib}
\usepackage{algpseudocode}
\def\qlearning{\textscTemporaryFix{S\texorpdfstring{\textsuperscript{3}}{3}Q-learning}}
\def\sqlearning{\textscTemporaryFix{S\texorpdfstring{\textsuperscript{4}}{4}Q-learning}}
\DeclareMathOperator{\TDerror}{TD}
\def\bonus{b}
\def\ControllerCovariance{\Sigmavar}
\newcommand{\Covariance}[1]{\Sigmavar^{(#1)}}
\newcommand{\QExpectedCovariance}[2]{\overline \Sigmavar_{#1}(#2)}

\newcommand{\ExpectedCovariance}[1]{\overline \Sigmavar^{(#1)}}

\def\ControllerPolicy{\pi}
\def\InQfunction{Q}
\newcommand{\Qfunction}[1]{Q^{(#1)}}
\newcommand{\Vfunction}[1]{V^{(#1)}}
\newcommand{\Policy}[1]{\pi^{(#1)}}
\newcommand{\ControllerNSwithces}[2]{s^{(#1)}_{#2}}

\def\ControllerCounter{m}
\def\ControllerTrigger{T}
\def\Qpar{\widehat \theta}
\newcommand{\Qtarget}[1]{\widehat Q^{tar}_{#1}}
\newcommand{\Qbest}[1]{\widehat Q^{\star}_{#1}}
\newcommand{\Vbest}[1]{\widehat V^{\star}_{#1}}

\newcommand{\BestPredictor}[3]{\theta^{#1,#2,#3}}
\def\QparTarget{\widehat \theta^{tar}}
\def\QparBest{\widehat \theta^{\star}}
\def\PolicyCover{\Pi}
\newcommand{\sqlearningPolicyCover}[1]{\Pi^{(#1)}}

\def\ControllerTotalCounter{m_\text{tot}}
\def\nSamplesLevel{n^{\level}}
\def\nSamplesEpoch{n_{\epoch}}
\newcommand{\nSamples}[1]{n^{(#1)}}
\def\level{\ell}
\def\epoch{e}
\def\nEpochs{e_{tot}}
\def\phase{p}
\def\nPhases{p_{tot}}
\def\episode{k}
\def\episodetot{K}

\newcommand{\UncertaintyParam}[1]{\alpha^{(#1)}}
\newcommand{\QUncertaintyParam}[2]{\overline \alpha_{#1}(#2)}
\newcommand{\MyUncertainty}{\mathscr{U}}

\newcommand{\UncertaintyFunction}[1]{\MyUncertainty^{(#1)}}
\newcommand{\ExpectedUncertaintyParam}[1]{\alpha^{(#1)}}
\newcommand{\ExpectedUncertaintyParamNoPhase}[1]{\alpha(#1)}
\def\ControllerUncertaintyParam{\alpha}

\def\startdistribution{\rho}

\newcommand{\CovarianceExpectedIncrement}[1]{\overline M^{(#1)}}
\newcommand{\nSamplesRand}[1]{N^{(#1)}_Q}
\newcommand{\RealV}[1]{V^{\Policy{#1}}}
\newcommand{\EmpiricalLoss}[3]{\widehat {\mathcal L}_{#1}(#2 \mid \mid #3)}

\newcommand{\InControllerTheta}[1]{\theta^{(#1)}}
\newcommand{\bonusnew}[1]{\bonus^{(#1)}}
\newcommand{\ERMerror}[1]{\Uerr^{(#1)}}
\def\Estart{\E_{\MyState \sim \startdistribution}}
\def\MainTerm{\ensuremath{\Term_4} }
\def\noise{\eta}

\def\sump{\sum_{\phase=1}^{\nPhases}}

\def\InformationGain{\widetilde d}

\newcommand{\Regretphase}[1]{\Regret^{(#1)}}

\newcommand{\cn}[1]{\varphi_{n}\left(#1\right)}
\newcommand{\csn}[1]{\varphi_{\sqrt{n}}\left(#1\right)}

\def\deltamaster{\delta_{\text{master}}}
\def\deltaPhase{\delta_{\text{phase}}}

\def\StationaryController{\pi}
\def\Uerr{\MyUncertainty}

\newcommand{\Aerr}{\Delta}

\newcommand{\Comparator}[2]{\Aerr^{#1}_{#2}}

\newcommand{\ControllerTriggeringValue}[1]{L_{\text{Trigger}}(#1)}
\def\TriggeringValue{L_{\text{trig}}}

\def\sumepisode{\sum_{\episode = 1}^\episodetot}

\def\PiIBE{\Pi^{(\text{lin})}}
\def\PiLR{\Pi^{(\text{exploration})}}
\def\QIBE{\Q^{(\text{lin})}}
\def\QLR{\Q^{(\text{all})}}

\def\prm{policy replay memory}

\def\MagicTrigger{\ControllerTrigger_{\text{Trig}}}

\def\cbonus{c}

\def\cboundedopt{c_o}

\def\MyState{S}
\def\MyAction{A}
\def\stateprime{\state'}
\def\actionprime{\action'}

\begin{document}

\begin{center}
  {\bf{\LARGE{
        Stabilizing $Q$-learning with Linear Architectures for Provably
Efficient Learning
}}}

  \vspace*{0.5in}

\begin{tabular}{lcl}
  Andrea Zanette$^\dagger$ && Martin J. Wainwright$^{\dagger, \star}$ \\
  \texttt{zanette@berkeley.edu} && \texttt{wainwrig@berkeley.edu}
\end{tabular}

\vspace*{0.10in}

\begin{tabular}{c}
  Department of Electrical Engineering and Computer
  Sciences$^{\dagger}$ \\
  Department of Statistics$^\star$ \\
  UC Berkeley, Berkeley, CA
\end{tabular}

  \vspace*{0.5in}

  \begin{abstract}
    The $Q$-learning algorithm is a simple and widely-used stochastic
approximation scheme for reinforcement learning, but the basic
protocol can exhibit instability in conjunction with function
approximation.  Such instability can be observed even with linear
function approximation. In practice, tools such as target networks and
experience replay appear to be essential, but the individual
contribution of each of these mechanisms is not well understood
theoretically.  This work proposes an exploration variant of the basic
$Q$-learning protocol with linear function approximation.  Our modular
analysis illustrates the role played by each algorithmic tool that we
adopt: a second order update rule, a set of target networks, and a
mechanism akin to experience replay.  Together, they enable state of
the art regret bounds on linear MDPs while preserving the most
prominent feature of the algorithm, namely a space complexity
independent of the number of step elapsed.  We show that the
performance of the algorithm degrades very gracefully under a novel
and more permissive notion of approximation error.  The algorithm also
exhibits a form of instance-dependence, in that its performance
depends on the ``effective'' feature dimension.

  \end{abstract}  
\end{center}


\section{Introduction}

The $Q$-learning algorithm~\cite{watkins1989learning} is a classical
and widely-used method for estimating optimal $Q$-value functions.  As
a stochastic approximation procedure for solving the Bellman fixed
point equation, it comes with strong convergence guarantees when
applied to tabular Markov decision processes
(e.g.,~\cite{tsitsiklis1994asynchronous,kearns1999finite,even2003learning,wainwright2019stochastic,li2021q}).
When combined with function approximation, however, the basic
$Q$-learning algorithm need not converge, and can exhibit instability.
This challenge has motivated various proposals for stabilizing the
updates.  Among other modifications, experience replay is one
ingredient that seems essential to state-of-the-art performance.  From
a theoretical point of view, however, these mechanisms are not well
understood.  This state of affairs leaves us with the following open
question: is it possible to derive a stable $Q$-learning procedure
with rigorous guarantees for a broad class of problem instances?

On one hand, recent work has unveiled information-theoretic barriers
applicable to any
algorithm~\cite{weisz2020exponential,zanette2020exponential,wang2020statistical,wang2021exponential,weisz2021tensorplan,foster2021offline}.
On the other hand, there exist several MDP models for which
sample-efficient RL is possible.  In particular, a recent line of
papers~\cite{krishnamurthy2016pac,jiang17contextual,sun2018model,zanette2020learning,jin2021bellman,du2021bilinear}
provide analyses of RL procedures for certain MDP classes, and provide
procedures that have polynomial sample complexity, albeit with
non-polynomial computational complexity.

The starting point of this paper is to study $Q$-learning in some
settings in which model-free algorithms\footnote{Sample efficient
learning algorithms have also been obtained for other settings, see
the papers~\cite{ayoub2020model,modi2020sample,modi2021model}.}  admit
polynomial-time implementation.  Examples include the class of
low-rank
MDPs~\cite{jin2020provably,zanette2020frequentist,agarwal2020flambe,agarwal2020pc,zanette2021cautiously},
and various generalizations
thereof~\cite{wang2019optimism,wang2020provably}.  Although the
underlying algorithms are polynomial-time, they can still require
prohibitive amounts of computation and storage in practical settings.
For instance, the memory requirement scales linearly with the amount
of experience collected, which limits its practical applicability.

The $Q$-learning algorithm is popular in applications precisely
because of its low computational complexity, as well as memory
requirements that do not scale with the iteration count.  Thus, we are
led to ask whether it is possible to devise a version of $Q$-learning
that is provably efficient when applied to low-rank MDPs.  We address
this question in the general exploration setting, so that a number of
challenges come into play, including credit assignment, moving
targets, and distribution shift.


\subsection{Our contributions}

The main contribution of this paper is to design and analyze a variant
of the $Q$-learning algorithm that is guaranteed to minimize regret
over the class of low-rank MDPs.  Three main ingredients are key in
our analysis: (1) a second-order update rule for improved statistical
efficiency; a set of target networks~\cite{mnih2015human} to stabilize
the updates, and most importantly, a replay mechanism called
\emph{policy replay}.  This mechanism is similar to experience replay
used in the deep RL literature (e.g.,~\cite{mnih2013playing}).  While
the second-order scheme and the target networks have been used in the
optimization and the RL literature before, the policy replay mechanism
is one key reinforcement learning contribution made in this paper.  It
stabilizes the learning process by eliminating the distribution shift
problem that naturally arises when converging to an optimal
controller.

Taken together, these algorithmic tools yield state-of-the-art regret
bounds on $\horizon$-horizon low-rank MDPs with $\dim$-dimensional
feature representations.  At the same time, they preserve one of the
most important features of $Q$-learning, namely a memory requirement
that---thanks to the policy replay mechanism---grows only
logarithmically with sample size.

We now provide an informal preview of our main result. We consider an
MDP with a finite action space of cardinality $|\ActionSpace|$, and
take (rescaling as needed) the optimal value function to be bounded in
$[0,1]$.  Letting $\nEpisodes$ the number of episodes elapsed, we have
the following:
\begin{theorem*}[Informal statement]
\label{thm:Intro}
There is a $Q$-learning algorithm that achieves the regret upper bound
$\ordertil(\horizon^2 \dim^{3/2} \sqrt{\nEpisodes})$ while using
$\ordertil(\dim^3\horizon^2)$ storage and per-step 
computational complexity $\order(\dim^2|\ActionSpace|)$.
\end{theorem*}
To our knowledge, this is the first regret bound for $Q$-learning with
any function approximator, which makes it the first algorithm with
bounded memory complexity for the considered setting.  The regret
bound is competitive with the state-of-the-art
results~\cite{jin2020provably}, in particular sub-optimal by a factor
of $\horizon$ in the regret bound.

\medskip

In this work, we also introduce a new notion of model
misspecification, one especially well-suited to the analysis of
temporal difference RL algorithms.  It is a much weaker requirement
than the $\ell_\infty$-norm bounds on mis-specification adoped in
prior analyses; instead, it involves the expected off-policy
prediction error.  To the best of our knowledge, this leads to the
mildest form of approximation error control for regret-minimizing
algorithms using temporal differences and function approximation.

Our results are also partially instance-dependent, in the sense that
we obtain faster rates for ``easier problems''.  In particular, we
show that the dimension $\dim$ in \cref{thm:Intro} can (mostly)
replaced by the \emph{effective dimension}, a quantity that can be
much smaller.  We are not aware of instance-dependent results of this
type when the algorithm is \emph{not} provided with side knowledge of
the problem structure.


\subsection{Relation to past work}

There is a long line of past work on $Q$-learning for tabular
problems, with results in both the asymptotic
settings~(e.g.,~\cite{watkins1992q,tsitsiklis1994asynchronous,jaakkola1994convergence,szepesvari1998asymptotic}),
as well as the non-asymptotic setting
(e.g.,~\cite{kearns1999finite,even2003learning,wainwright2019stochastic,li2021q}).
Other work on $Q$-learning in tabular problems has derived regret
bounds that are also near-optimal~\cite{jin2018q,zhang2020almost}.

It is well known that once function approximation is introduced, then
the $Q$-learning algorithm may diverge~\cite{baird1995residual}.  Such
divergence does not occur in certain special cases, including when the
dynamics are restricted to induce similar directions in feature
space~\cite{melo2008analysis}, or the function approximators are
$\ell_\infty$-contractive in an appropriate sense
(e.g.~\cite{gordon1995stable}).  Related results are presented in the
papers~\cite{tsitsiklis1997analysis,perkins2002existence,mehta2009q,liu2020finite,bhandari2018finite,lakshminarayanan2018linear}.
In contrast, our analysis does not impose such conditions.  We also
note that there is some recent analysis of $Q$-learning with deep
neural networks~\cite{fan2020theoretical} that leverages connectionms
to neural fitted
$Q$-iteration~\cite{riedmiller2005neural,munos2008finite}; see also
the papers~\cite{cai2019neural,carvalho2020new}.

Some of the algorithmic techniques used in this work---specifically,
the use of target networks and experience replay---are believed to be
essential to recent empirical successes in reinforcement learning.
Experience replay was introduced by Lin~\cite{lin1992self}, and
popularized more widely by the influential
paper~\cite{mnih2013playing}.  To be clear, our replay mechanism
differs in that it does not store past rewards and transitions; this
fact is essential to maintaining low memory complexity.  Our replay
mechanism is related to the policy cover
mechanism~\cite{agarwal2020pc,zanette2021cautiously}, but differs in
that it needs to store high performance policies, and it is not used
as starting distribution for policies roll-outs.  As for target
networks, they have also been a core component of past empirical
successes~\cite{mnih2015human}.

We note that recent work by Agarwal et al.~\cite{agarwal2021online}
also shows the importance of forms of experience replay, in
establishing a result related to our \cref{thm:Q-learning}.  Our work
shows that experience replay is not needed when the controller is
stationary.  Indeed, our primary contribution is in the exploration
setting (cf. \cref{thm:Sequoia}), whose literature we discuss next.  A
related and concurrent work in the exploration setting is
\cite{liu2022provably}.

To the best of our knowledge, this paper constitutes the first
analysis of an exploratory form of $Q$-learning combined with function
approximation.  It can be compared with the work of Jin et
al.~\cite{jin2020provably}, who proved guarantees for exploration
based on a form of least-squares value iteration (LSVI) with optimism
for the class of low-rank MDPs.  However, their algorithm has a space
complexity that grows linearly with time, and the approximation error
requirements are expressed via sup-norm ($\ell_\infty$) bounds.
Better approximation error requirements with respect to a fixed
comparator are given by policy gradient
methods~\cite{agarwal2020pc,zanette2021cautiously}, whose memory
complexity still grows with the required accuracy\footnote{In the
paper~\cite{agarwal2020pc}, the policy cover grows linearly with the
iteration count while the method~\cite{zanette2021cautiously} needs to
store past trajectories to perform data reuse.}.  Our work shows that
attractive approximation error guarantees are not unique to policy
gradient algorithms: temporal difference methods also inherit
favorable---albeit different---guarantees.  While this has recently
been noted in the offline setting, such guarantees were enabled by a
dataset generated from a stationary
distribution~\cite{xie2021bellman}, as opposed to a reactive
controller~\cite{zanette2021provable}, which is the standard case in
the exploration setting.

Finally, to our knowledge none of algorithms discussed so far inherit
instance-dependent regret bounds while being agnostic to the setting.
The bulk of past instance-dependent results correspond to tabular
problems
(e.g.,~\cite{zanette2019tighter,zanette2019b,simchowitz2019non,yin2021towards,tirinzoni2021fully,al2021adaptive,xu2021fine,wagenmaker2021beyond,yang2021q,
  KhaXiaWaiJor21, XiaKhaWaiJor22}; a few exceptions include the
logarithmic regret bounds given in the paper~\cite{he2021logarithmic}
and the recent paper~\cite{wagenmaker2021firstorder}, as well as some
partially instance-dependent results on kernel
LSTD~\cite{DuaWaiWan21}.  Other studies related to $Q$-learning
include the
papers~\cite{li2022note,yan2022efficacy,shi2022pessimistic,santos2021understanding,xu2021constraints}.


\section{Background and problem formulation}

We begin by providing background and describing some structural
assumptions related to our analysis.

\subsection{Finite-horizon Markov decision proceses}

In this paper, we focus on finite-horizon Markov decision processes;
see the standard
references~\cite{puterman1994markov,bertsekas1996neuro} for more
background and detail. A finite-horizon MDP is specified by a positive
integer $\horizon$, and events take place over a sequence of stages
indexed by the time step $\hstep \in [\horizon] \defeq \{1, \ldots,
\horizon \}$.  The underlying dynamics involve a state space
$\StateSpace$, and are controlled by actions that take values in some
action set $\ActionSpace$.  In the analysis of this paper, the state
space is allowed to be arbitrary (discrete or continuous), but we
restrict to a finite action space.

For each time step $\hstep \in [\horizon]$, there is a reward function
$r_\hstep: \StateSpace \times \ActionSpace \rightarrow \R$, and for
every time step $\hstep$ and state-action pair $(\state, \action)$,
there is a probability transition function $\Pro_\hstep(\cdot \mid
\state, \action)$.  When at horizon $\hstep$, if the agent takes
action $\action$ in state $\state$, it receives a random reward drawn
from a distribution $R_\hstep(\state, \action)$ with mean
$r_\hstep(\state, \action)$, and it then transitions randomly to a
next state $\state'$ drawn from the transition function
$\Pro_\hstep(\cdot \mid \state, \action)$.

A policy $\pi_\hstep$ at stage $\hstep$ is a mapping from the state
space $\StateSpace$ to the action space $\ActionSpace$.  Given a full
policy $\pi = (\pi_1, \ldots, \pi_\horizon)$, the state-action value
function at time step $\hstep$ is given by
\begin{align}
Q^{\pi}_\hstep(\state, \action) & = r_\hstep(\state, \action) +
\E_{\MyState_\ell \sim \pi \mid (\state, \action)} \sum_{\ell = \hstep
  + 1}^{\horizon} r_{\ell}(\MyState_{\ell},
\pi_{\ell}(\MyState_\ell)),
\end{align}
where the expectation is over the trajectories induced by $\pi$ upon
starting from the pair $(\state, \action)$. When we omit the starting
state-action pair $(\state, \action)$, the expectation is intended to
start from a fixed state denoted by $\state_1$.  Any policy is
associated with a value function $V^{\pi}_\hstep(\state) =
Q^{\pi}_\hstep(\state, \pi_\hstep(\state))$, along with a Bellman
evaluation operator
\begin{align*}
\T^{\pi}_\hstep(Q_{h+1})(\state, \action) = r_\hstep(\state, \action)
+ \E_{\MyState' \sim \Pro_\hstep(\state, \action)} \E_{\MyAction' \sim
  \pi} Q_{h+1}(\MyState',\MyAction').
\end{align*}
	
Under some regularity
conditions~\cite{puterman1994markov,shreve1978alternative}, there
always exists an optimal policy $\pi^\star$ whose value and
action-value functions achieve the suprema
\begin{align*}
\Vstar_\hstep(\state) = V^{\pi^\star}_\hstep(\state) = \sup_{\pi}
V^{\pi}_\hstep(\state), \quad \mbox{and} \quad \Qstar_\hstep(\state,
\action ) = Q^{\pi^\star}_\hstep(\state, \action) = \sup_{\pi}
Q^{\pi}_\hstep(\state, \action).
\end{align*}
uniformly over all states and actions.  We use $\E_{\pi}[\phi_\hstep]
\defeq \E_{(\MyState_\hstep, \MyAction_\hstep) \sim \pi}
       [\phi_\hstep(\MyState_\hstep, \MyAction_\hstep)]$ to denote the
       expected feature vector at timestep $\hstep$. \\

We analyze algorithms that produce sequences of policies $\{ \pi^1,
\ldots, \pi^\nEpisodes \}$, and for any such sequence, we define the
regret
\begin{align}
\label{eqn:Sequoia}
\Regret(\nEpisodes) & \defeq \sumepisode
\E_{\state_1\sim\startdistribution}\left( \VstarFH{1} -
\VpiagentFH{1}{\episode} \right)(\state_{1}).
\end{align}
Whenever we have a sequence $n^1,\dots,n^k$ of values we denote with
$n^{1:k} = \sum_{i=1}^k n^i$ their sum.

\subsection{Structural conditions}

Let now us lay out some assumptions on the MDPs and function
approximation schemes.

\subsubsection{Linear function approximations}

For each $\hstep \in [\horizon]$, let $\phi_\hstep: \StateSpace
\times \ActionSpace \mapsto \R^\dim$ be a given feature map.
Throughout this paper, we assume the uniform boundedness condition
\begin{align}
\label{EqnBoundedFeatures}
\sup_{\state, \action} \|\phi_\hstep(\state, \action)\|_2 \leq 1
\qquad \mbox{for all $\hstep \in [\horizon]$.}
\end{align}
For a given parameter vector $\theta_\hstep \in \R^\dim$, define the
function $f_{\hstep,\theta}(\state, \action) \defeq
\smlinprod{\phi_\hstep(\state, \action)}{\theta_\hstep}$.  With a
slight abuse of notation, given a partitioned vector $\theta =
(\theta_1, \ldots, \theta_\horizon) \in (\R^\dim)^\horizon$, we use
the shorthand $f_\theta = (f_{1, \theta_1}, \ldots, f_{\horizon,
  \theta_\horizon})$ for the associated collection of functions.

In this paper, we study algorithms that produce linear functions in
the class
\begin{subequations}
\begin{align}
\label{eqn:FunctionalSpace}
\QIBE & \defeq \Big \{ f_\theta \mid \|\theta_\hstep\|_2 \leq 1 \quad
\mbox{for all $\hstep \in [\horizon]$} \Big \}.
\end{align}
Note that the bounded feature map
condition~\eqref{EqnBoundedFeatures}, in conjunction with the
Cauchy-Schwarz inequality, implies that
\begin{align}
\|f_{\hstep, \theta_\hstep}\|_\infty & = \sup_{\state, \action}
|f_{\hstep, \theta_\hstep}(\state, \action)| \leq 1 \qquad \mbox{for
any $f_\theta \in \QIBE$.}
\end{align}
Consequently, the function class is contained with the larger class of
action-value functions \mbox{$(\state, \action) \mapsto
  Q_\hstep(\state, \action)$} that are uniformly bounded in
sup-norm---more precisely, the class
\begin{align}
\QLR \defeq \{ (Q_1, \ldots Q_\horizon) \mid
\|Q_\hstep\|_\infty \leq 1 \quad \mbox{for all $\hstep \in
[\horizon]$} \}.
\end{align}
\end{subequations}
The definitions above can be specialized for a specific timestep
$\hstep$ in a natural way, in which case we denote the corresponding
function spaces by $\QIBE_\hstep$ and $\QLR_\hstep$.

\subsubsection{Bellman conditions}

Our work covers both the settings with low inherent Bellman error
(e.g., \cite{munos2008finite,zanette2020learning}) as well as low-rank
MDPs (e.g., \cite{yang2020reinforcement,jin2020provably}), which we
introduce next. In both cases we assume that $\QIBE_{\horizon+1} =
\QLR_{\horizon+1} = \{ 0\}$.
\begin{assumption}[Bellman closure]
\label{asm:InherentBellman}
We say that an MDP and a feature representation $\phi$ have \emph{zero
inherent Bellman error} if for each $\hstep \in [\horizon]$ and any
$Q_{\hstep+1} \in \QIBE_{\hstep+1}$, there exists $Q_\hstep \in
\QIBE_{\hstep}$ such that $Q_\hstep = \T_\hstep Q_{\hstep+1}$.
\end{assumption}

\begin{assumption}[Low-Rank]
\label{asm:LowRank}
An MDP is low rank with respect to the feature representation $\phi$
if for each $\hstep \in [\horizon]$, the following holds:
\begin{align*}
\mbox{$\forall Q_{\hstep+1}\in \QLR_{\hstep+1}$, there exists
$Q_\hstep \in \QIBE_\hstep$ s.t.  $Q_\hstep = \T_\hstep
Q_{\hstep+1}$.}
\end{align*}
\end{assumption}
\noindent It can be shown that the class of low-rank MDP models is
strictly contained within the class of MDPs with zero inherent Bellman
error; see the paper~\cite{zanette2020learning} for further details.

\medskip

\myparagraph{Model misspecification} When the representation
conditions do not exactly hold, we need to measure model
misspecification. With this aim, we introduce two definitions of model
misspecification that are appropriate for RL with temporal difference
methods.  The first one measures the violation of
\cref{asm:InherentBellman} with respect to a stationary external
controller, while the second one measures the violation with respect
to \cref{asm:LowRank} when a single stationary controller is not
available.

Before stating the definitions, let us introduce some more notation
and terminology along with their motivation.  Let $\pi$ be a policy
that generates a dataset used to fit a predictor.  Using the data
generated by $\pi$, we will make predictions about a target policy $
\pibar$ which could be arbitrary.  The predictor that we seek should
fit $\T Q'$ where $Q'\in\QIBE$ or $Q'\in\QLR$ depending on whether we
seek to quantify the violation of \cref{asm:InherentBellman} or
\cref{asm:LowRank}, respectively.  Accordingly, define the population
minimizer $\BestPredictor{\pi}{Q'}{\hstep}$ along $\pi$ with $Q'$ as
next state value function as
\begin{align}
\label{eqn:BestPredictor}
\BestPredictor{\pi}{Q'}{\hstep} \defeq \argmin_{\theta \in \B} 
\E_{(\MyState_\hstep,\MyAction_\hstep) \sim \pi}
\Big\{ \innerprod{ \phi_\hstep(\MyState_\hstep,\MyAction_\hstep)}{\theta}
- (\T_\hstep Q')(\MyState_\hstep,\MyAction_\hstep)) \Big\}^2.
\end{align}

Let us now state a definition of model misspecification that measures
the violation with respect to \fullref{asm:InherentBellman} whenever
there exists an external stationary controller $\pi$.  This definition
involves a non-negative error term $\terr \geq 0$ referred to as
\emph{transfer error}.

\begin{definition}[Model Misspecification w.r.t. Bellman Closure]
\label{def:ModelMisspecificationIBE}
An MDP and a feature map $\phi$ are $ \terr$-misspecified with
respect to the Bellman closure condition and the stationary policy $\pi$ if
for any policy\footnote{
When we measure the errors with respect to \cref{asm:InherentBellman},
it would be enough to consider policies $(\pi, \pibar)$ in the class
\begin{align*}
\PiIBE \defeq \Big \{ \pi \mid \state \mapsto \arg \max_\action
\smlinprod{\phi(\state,\action)}{\theta} \mid \|\theta \|_2 \leq 1
\Big \} \cup \{ \pistar \}
\end{align*}
}
 $\pibar$ and action-value function $Q' \in
\QIBE$ the best on policy fit 
$Q_\hstep: (\state,\action) \mapsto 
\innerprod{\phi_\hstep(\state,\action)}{\BestPredictor{\pi}{Q'}{\hstep}} $ 
along $\pi$ satisfies the bound
\begin{align}
\label{eqn:TransferError}
\Big|\sumh
\E_{(\MyState_\hstep, \MyAction_\hstep) \sim \pibar} \Big [
Q_\hstep(\MyState_\hstep, \MyAction_\hstep) - (\T_\hstep
Q'_{\hstep+1})(\MyState_\hstep, \MyAction_\hstep) \Big] \Big| \leq \terr.
\end{align}
\end{definition}
In summary, \cref{def:ModelMisspecificationIBE} measures the average Bellman error
that arises when evaluating the predictor fit on the controller's distribution
along other distributions.
This is a significantly more generous requirement than $\ell_\infty$ model
misspecification, and is algorithm-independent.
Notice that the expectation is
inside the absolute value. 
We conclude by presenting an extension of
\cref{def:ModelMisspecificationIBE}, one that applies to the
exploration setting where there is no single stationary controller
that generates the dataset.
\begin{definition}[Model Misspecification w.r.t. Low Rank]
\label{def:ModelMisspecificationLR}
An MDP and a feature map $\phi$ are $\terr$-misspecified with
respect to the low rank condition if for any two policies\footnote{
When we refer to \cref{asm:LowRank}, it is enough to consider policies in the 
class 
\begin{align*}
\PiLR \defeq \Big \{ \pi \mid \state \mapsto \argmax_\action \left
\{ \smlinprod{\phi(\state,\action)}{\theta} + \alpha \|
\phi(\state,\action) \|_M \right \} \mid \| \theta \|_2 \leq 1,\; 0
\preccurlyeq M \in \R^{\dim\times\dim}, 0 \leq \alpha \in \R
\Big \} \cup \{ \pistar \}.
\end{align*}
} $\pi,\pibar$ and action value function $Q'\in\QLR$, the best on policy fit 
$Q_\hstep: (\state,\action) \mapsto 
\innerprod{\phi_\hstep(\state,\action)}{\BestPredictor{\pi}{Q'}{\hstep}} $ 
satisfies the
bound~\eqref{eqn:TransferError}.
\end{definition}
\noindent The primary distinction
between~\cref{def:ModelMisspecificationIBE}
and~\cref{def:ModelMisspecificationLR} is that the latter needs to
hold when $Q'\in \QLR$ instead of just $Q'\in \QIBE$.


\section{Algorithms}
\label{sec:Qlearning}

This section is devoted to a description of the $Q$-learning
procedures analyzed in this paper.  We begin by providing some
intuition for our algorithms in Section~\ref{SecIntuition}.
Section~\ref{SecS3} is devoted to the description of \emph{Stabilized,
Second-Order, Streaming $Q$-learning} algorithm, or \qlearning{} for
short.  It corresponds to a stabilized and streaming form of
$Q$-learning that estimates the optimal policy based on data drawn
from some fixed (stationary) controller policy.  We use this algorithm
as a building block for the more sophisticated algorithm described in
Section~\ref{sec:Sequoia}, which allows for the data-generating policy
to also change, essential to obtaining an overall scheme with low
regret.  We refer to this procedure as \emph{Sequentially Stabilized
Second-order Streaming $Q$-learning}, or \sqlearning{} for short.

\subsection{Some intuition}
\label{SecIntuition}

Let us begin by providing some intuition for the algorithms that are
proposed and analyzed in this paper.  When the basic form of
$Q$-learning is implemented with linear function approximation, the
updates are performed directly on the parameter $\theta$ associated
with the linear representation.  Upon observing the tuple
$(\state_\hstep, \action_\hstep, \reward_\hstep, \stateprime_\hstep)$,
representing the experienced state, action, reward and successor state
at level $\hstep$, the update rule for a user defined learning rate
$\alpha \in\R$ takes the familiar form
\begin{align}
\label{eqn:BasicQupdate}
\theta_\hstep \leftarrow \theta_\hstep - \alpha \Big[
  \underbrace{\smlinprod{\phi_\hstep(\state_\hstep,
      \action_\hstep)}{\theta_\hstep} - \reward_\hstep -
    \max_{\actionprime} \smlinprod{\phi_{\hstep+1}(\stateprime_\hstep,
      \actionprime)}{\theta_{\hstep+1}}}_ {\text{TD error}} \Big]
\phi_\hstep(\state_\hstep,\action_\hstep).
\end{align}
Although $Q$-learning is a form of stochastic approximation, as are
stochastic gradient methods, the above update is not equivalent to
stochastic gradient.  However, for the purposes of analysis, it is
useful to consider some restrictions under which it can be related to
a stochastic gradient method.

For a moment, let us additionally assume that (a) the next timestep
parameter $\theta_{\hstep+1}$ is never updated, and (b) the tuple
$(\state_\hstep,\action_\hstep,\reward_\hstep,\stateprime_\hstep)$ is
drawn from a stationary distribution.  When these conditions are met,
the update~\eqref{eqn:BasicQupdate} corresponds to a stochastic
gradient update as applied to the squared loss
\begin{align}
\label{eqn:Qloss}
\theta_\hstep \mapsto \E_{(\MyState_\hstep, \MyAction_\hstep,
  \Reward_\hstep, \MyState'_\hstep)} \Big[\underbrace{
    \vphantom{\max_{\actionprime}
      \smlinprod{\phi_{\hstep+1}(\MyState'_\hstep,\actionprime)}{\theta_{\hstep+1}}}
    \smlinprod{\phi_\hstep(\MyState_\hstep,\MyAction_\hstep)}{\theta_\hstep}}_{\text{predictor}}
  - \; \underbrace{\big( \Reward_\hstep + \max_{\actionprime}
    \smlinprod{\phi_{\hstep+1}(\MyState'_\hstep,
      \actionprime)}{\theta_{\hstep+1}} \big)}_{\text{fixed target
      function}} \Big]^2,
\end{align}
where the expectation is over the stationary distribution that
generates the data.  This is an algorithm that we know how to analyze.

With this perspective in place, our high-level idea is to enforce
these two conditions---namely, a fixed target and a stationary
distribution for drawing samples.  However, so as to be able to
estimate an optimal policy while incurring low regret, the next
state-value function and the sampling distribution cannot be ``locked
in'' forever, but instead need to evolve with time.  The core
algorithmic contribution of this paper is the design of a device to
periodically update the next-state value function and the sampling
distribution so as to allow convergence to an optimal controller in a
stable way.  In addition, we use a second-order update in place of the
first-order scheme~\eqref{eqn:BasicQupdate} so as to achieve improved
statistical efficiency.

We first describe an algorithm for the controlled setting, in which
stream of states, actions, rewards and transitions are generated by a
stationary controller.  In this case, only the next-state action value
function needs to be updated periodically, because the distribution
that generates the experience is fixed.  Next, we design a
meta-algorithm that performs exploration while additionally ensuring
that the $Q$-learning update rule is fed with data from a stationary
controller.


\subsection{\qlearning{}}
\label{SecS3}

In this section, we introduce the \emph{Stabilized, Second-Order,
Streaming $Q$-learning} algorithm, or \qlearning{} for short. The
algorithm takes as input a stationary controller policy $\pi$ that
generates a stream of states, actions, rewards and transitions. Target
networks are used to stabilize the value function updates, which are
performed via a second-order update rule for improved statistical
efficiency. This is a streaming algorithm, meaning that each sample is
immediately processed and then discarded.


\medskip

\myparagraph{Learning mechanics} The \qlearning{} algorithm is
detailed in \cref{alg:S3Qlearning}. The algorithm proceeds over a
sequence of epochs, denoted by $\epoch$ in the algorithm.  Each epoch
is handled by the outermost ``while'' loop.  Within each epoch, the
algorithm sequentially updates the target networks\footnote{
In this work we refer to the next-timestep linear approximator as to
`target network' for consistency with some of the $Q$-learning literature.
However, notice that our `networks' are linear.} $\Qtarget{\hstep}$
at each level $\level$ proceeding backward from $\level=\horizon$ to
$\level=1$.  This order of updates ensures that the next-timestep
($\level+1$) target network is always up to date to compute the
bootstrapped $Q$ values needed at level $\level$ to compute the
temporal difference (TD) error (see \cref{line:TDError}
and~\cref{eqn:TDerror}).  When the update has completed in every level
$\level\in [\horizon]$, the target networks are stored in the
predictor $\Qbest{}$, which is the one that the algorithm considers to
be the ``best'' estimate of $\Qstar$. At this point, a new epoch
begins.

Let us now describe the update rule in
\cref{line:TDError,line:UpdateRule}.  At each timestep $\hstep$ the
algorithm observes a tuple of state, action, reward and transition
$(\state_\hstep,\action_\hstep,\reward_\hstep,\state_{\hstep+1})$ and
uses them to update $\Qpar_\hstep$. To be clear, $\Qpar_\hstep$ is
associated to a network different from the target network
$\Qtarget{\hstep}$ for that timestep. To perform the update, the
algorithm first computes the temporal difference error
\begin{align}
\label{eqn:TDerror}
\TDerror_h \defeq \reward_h + \max_{\actionprime}
\Qtarget{\hstep+1}(\state_{\hstep+1},\actionprime) -
\innerprod{\phi_\hstep(\state_h,\action_h))}{\Qpar_h}
\end{align} 
in \cref{line:TDError} and then updates the network parameter
$\Qpar_\hstep$ using a second order update rule (in place of
\cref{eqn:BasicQupdate}) together with the empirical covariance
$\Sigma^{-1}_h$, see \cref{line:UpdateRule}.  Such update rule
effectively minimizes the least-squares criterion~\eqref{eqn:Qloss},
as it coincides with the Sherman-Morrison rank one update.

The stopping condition in line \ref{line:stoppingcond} can be any
arbitrary stopping time; without it, the algorithm will simply keep
running indefinitely.

\begin{algorithm*}[htb]
\caption{\textscTemporaryFix{\qlearning}}
\label{alg:S3Qlearning}
\begin{algorithmic}[1]
\State \textbf{Input:} Controller $\pi$, (optional) stopping condition, (optional) bonus function $\bonus$ 
\State$\Qtarget{\level}(\cdot,\cdot) = 0, \; \forall \level \in [\horizon+1]; \; \epoch = 0 $ \Comment{Initialize target network and epoch counter}   
\While{True}
\State $\epoch = \epoch+1$ \Comment{New epoch begins}
\For{level $\level = H,H-1,\dots,2,1$}
\State $\Qpar_\level = 0; \; \Sigma_\level = \lambdareg\Identity; $ \Comment{Initialize network and covariance}
\For{$\nSamplesLevel = 1,\dots, 2^\epoch$}
\IfThen{\textscTemporaryFix{Stopping Condition}}{\textbf{return} $\Qbest{}$} \label{line:stoppingcond}
\State $\state_1 \sim \startdistribution$ \Comment{Get start state}
\For{timestep $\hstep = 1,2,\dots,\horizon$}
\State  Play $\action_\hstep =\pi_\hstep(\state_h)$ and get $(\reward_h,\state_{h+1}); \; \phi_h \defeq \phi_\hstep(\state_\hstep,\action_\hstep)$\Comment{Play and advance}
\State  $\TDerror_h = \reward_h + \max_{\actionprime}  \Qtarget{\hstep+1}(\state_{\hstep+1},\actionprime) -  \innerprod{\phi_\hstep}{\Qpar_h}$,  \label{line:TDError} \Comment{Compute TD error}
\State $\Qpar_\hstep \leftarrow \Qpar_\hstep + \frac{\Sigma^{-1}_\hstep \phi_h \TDerror_h  }{1 + \| \phi_h \|^2_{\Sigma_\hstep^{-1}}}; \qquad \Sigma^{-1}_h \leftarrow \Sigma^{-1}_h - \frac{\Sigma_{\hstep}^{-1}\phi_h\phi_h^\top\Sigma^{-1}_\hstep}{1+\| \phi_h\|^2_{\Sigma^{-1}_\hstep}}$  \label{line:UpdateRule} \Comment{Update network and covariance}
\EndFor
\EndFor
\State $\QparTarget_\level = \min_{\theta \in \B} \| \theta - \widehat \theta_\level \|^2_{\Sigma_\level}$ \Comment{Project parameter} 
\State $\Qtarget{\level}(\cdot,\cdot)\leftarrow \innerprod{\phi_\level(\cdot,\cdot)}{\QparTarget_\level} $ or $\Qtarget{\level}(\cdot,\cdot)\leftarrow \minbetween{1}{\innerprod{\phi_\level(\cdot,\cdot)}{\QparTarget_\level} + \bonus_\level(\cdot,\cdot)}$\label{line:TargetUpdate} 
\Comment{Update target network}
\EndFor
\State $\Qbest{} \leftarrow \Qtarget{}$ \Comment{Save best approximator}
\label{line:Qbest}
\EndWhile
\end{algorithmic}
\end{algorithm*}

%


\subsection{\texorpdfstring{\sqlearning{}}{Sequoia}}
\label{sec:Sequoia}

When the stream of data is generated by a controller that is
converging to an optimal one---a necessary condition to obtain low
regret---the experience it generates is no longer stationary.  We will
now introduce a simple device, the \emph{policy replay} mechanism,
that allows the controller to evolve with time while ensuring that
there is no distribution shift during the $Q$-learning updates.  It
leads to an algorithm that converges to an optimal controller under
some assumptions.  This is achieved by \emph{sequentially} invoking
\qlearning{} using stationary controllers that are increasingly more
optimal; the resulting algorithm is called \emph{Sequentially
Stabilized Second-order Streaming $Q$-learning}, or \sqlearning{} for
short, and is detailed in \cref{alg:Sequoia}.


\medskip

\myparagraph{Policy replay for experience replay} The policy replay
mechanism generates new experience using past policies.  The past
policies are stored in the \emph{\prm{}} $\PolicyCover \defeq
\{(\pi_\iindex,n_\iindex)\}_{\iindex=1}^{\phase}$, which contains a
set of policies $\pi_\iindex$ associated to a number of samples
$n_\iindex$.  The policy replay mechanism extracts a stationary
mixture policy from $\PolicyCover$, defined as the controller that
plays each policy $\pi_\iindex$ with probability proportional to
$n_\iindex$ for the full episode.  Such mixture policy is taken as
stationary controller to invoke \qlearning{}, along with a suitable
exploration bonus to produce optimistic $Q$-values that guide the
exploration.  The stopping condition to be used in
\cref{line:stoppingcond} in \cref{alg:S3Qlearning} is the number of
trajectories $c\horizon\ControllerTotalCounter$ for an appropriate
constant $c$, see line \cref{line:SequoiaCallsQlearning} in
\cref{alg:Sequoia}.

The policy replay mechanism is similar in purpose to experience replay
where the experience $\{(\state_\iindex,\action_\iindex,
\reward_\iindex,\stateprime_\iindex)\}$ generated so far is stored and
used to retrain the network.  However, unlike experience replay, the
\prm{} does not store the full dataset and instead just contains a
`recipe' for generating a statistically similar dataset by re-playing
past policies.  In this way, the memory complexity does not grow with
the number of iterations beyond a mild logarithmic term, making our
algorithm truly streaming.  And while the policy replay mechanism
requires additional samples, the regret remains well controlled
because the policies in $\PolicyCover$ are progressively more and more
near-optimal: in the limit, the policies in the \prm{} generate
samples with vanishing regret.

Having described the policy replay mechanism, we can now illustrate
how \sqlearning{} conducts exploration using such device.


\medskip

\myparagraph{Learning mechanics} The~\sqlearning{} algorithm proceeds
in phases which are indexed by $\phase$ inside the algorithm. At the
beginning of each phase, the algorithm invokes the~\qlearning{}
subroutine with a controller $\pi_{\text{Control}}$ that is a mixture
policy among those in the current \prm{} $\PolicyCover$, as described
in the prior paragraph.  The \qlearning{} procedure then returns an
optimistic action-value function estimate $\InQfunction{}$ from which
the greedy (optimistic) policy $\ControllerPolicy$ is extracted in
\cref{line:GreedyPolicy} of \cref{alg:Sequoia}.

The \sqlearning{} algorithm then proceeds to collect trajectories from
the greedy policy $\pi$ until ``sufficient progress'' is made. In
order to measure the progress, the agent maintains the accumulator
$\ControllerTrigger$ and updates it in
\cref{line:AccomulatorIncrement}. Once $\ControllerTrigger$ is larger
than a certain value (see \cref{line:NewPolicyTrigger}), the procedure
has made sufficient progress on the current data, so that
$\InQfunction{}$ network should be updated with fresh data. The
triggering condition in~\cref{line:NewPolicyTrigger} is essentially
equivalent to checking that the determinant of the cumulative
covariance has doubled with respect to that of the prior epoch (cf. a
similar condition for linear bandits~\cite{Abbasi11}).  When the
determinant doubles, the agent has acquired sufficient information and
a new policy may be computed.  An important difference here is that
the determinant should refer to the expected cumulative covariance,
which is unknown, and such determinant ratio must thus be estimated
from data; our accumulator $\ControllerTrigger$ performs such task.

In order to update the $\InQfunction{}$ network, \sqlearning{}
constructs a new bonus function and adds the current greedy policy
$\ControllerPolicy$ to the \prm{} (together with the number of
trajectories that should be generated from such policy). A new phase
can now begin with a call to \qlearning{} using a newly constructed,
more optimal controller and smaller bonus function.

\begin{algorithm*}[htb]
\caption{\sqlearning}
\label{alg:Sequoia}
\begin{algorithmic}[1]
\State \textbf{Input:} Bonus function $\bonus$, update trigger
$\MagicTrigger = \Theta(\log \frac{\phase \nEpisodes}{\delta})$
\State
$\PolicyCover_\hstep = \emptyset; \; \forall \hstep \in [\horizon]$
\Comment{Initialize \prm{} }  \For{phase $\phase
  =1,2,\dots$} \State $\ControllerTotalCounter = \sum_{\jindex =1}^{\phase-1}
\ControllerCounter_\jindex$ for $(\ControllerPolicy_\jindex,
\ControllerCounter_\jindex) \in \PolicyCover$ \Comment{Get total \#
  trajectories to simulate} \State $\pi_{\text{Control}} \defeq$ at
the start of the episode play $\pi_\jindex$ with probability $
\ControllerCounter_\jindex/\ControllerTotalCounter, \; \forall \jindex
\in [\phase-1]$ \Comment{Define controller} \State $(\InQfunction,
\ControllerCovariance^{ref} )
\leftarrow$\qlearning$(\pi_{\text{Control}},c\horizon\ControllerTotalCounter,\bonus), \; c \in \R$
\label{line:SequoiaCallsQlearning}
\Comment{Get optimistic $Q$ values} \State$\ControllerPolicy_\hstep
(\cdot)= \argmax_{\action} \InQfunction_\hstep(\cdot,\action), \;
\forall \hstep \in [\horizon]$ \label{line:GreedyPolicy}
\Comment{Extract greedy policy} 
\State $\ControllerCovariance = \ControllerCovariance^{ref}, \; \ControllerCounter = 0,
\;\ControllerTrigger_\hstep=0,\; \forall \hstep \in [\horizon]$
\Comment{Initialize \# trajectories and trigger value}
\Repeat \label{line:NewPolicyRepeat} \State $\state_1 \sim
\startdistribution$\Comment{Get start state} \For{$\hstep =
  1,2,\dots,\horizon$} 
  \State Play $\action =
\ControllerPolicy_\hstep(\state); \; \ControllerCounter = \ControllerCounter + 1$ \Comment{Play and increment counter} \State
$\ControllerTrigger_\hstep \leftarrow
\ControllerTrigger_\hstep + \| \phi_\hstep(\state,\action)
\|^2_{(\ControllerCovariance^{ref}_{\hstep})^{-1}}; \ControllerCovariance_{\hstep} \leftarrow
\ControllerCovariance_{\hstep} + \phi_\hstep(\state,\action)
\phi_\hstep(\state,\action)^\top
$ \label{line:AccomulatorIncrement} \Comment{Increment accumulator and covariance}
  \State
   Get next state $\state^+$; $\state
\leftarrow \state^+$ \Comment{Advance}
\EndFor \Until $\exists \hstep \in [\horizon]$ such
that $\ControllerTrigger_\hstep \geq
\MagicTrigger$ \label{line:NewPolicyTrigger} \State $\PolicyCover
\leftarrow \PolicyCover\cup\{ (\ControllerPolicy,\ControllerCounter)
\}; \; \bonus_\hstep(\cdot,\cdot) = \ControllerUncertaintyParam_\hstep
\| \phi(\cdot,\cdot)
\|_{\ControllerCovariance_\hstep^{-1}}$ \label{line:bonusS4}
\Comment{Update \prm{} and bonus} \EndFor
\end{algorithmic}
\end{algorithm*}


\section{Main results}
\label{SecMain}

We now turn to the statement of our main results, along with
discussion of some of their consequences.  We begin in
Section~\ref{SecBonus} by describing the form of the bonus function
used in our algorithms, along with the effective dimension that
appears in our bounds.  Section~\ref{SecS4Guarantees} is devoted to
our main result---namely, a performance guarantee
for the \sqlearning{} algorithm.  In Section~\ref{SecS3Guarantees},
we elaborate upon the guarantees for the \qlearning{} algorithm
that underlie our main result.


\subsection{Bonus function and effective dimension}
\label{SecBonus}

We first describe the bonus function involved in the algorithm, along
with a notion of effective dimension that arises in the analysis.

\paragraph{Bonus function:}

The bonus function used in phase $\phase$ of the updates takes the
form
\begin{align}
\label{eqn:Bonus}
\bonus_\hstep(\cdot,\cdot) \defeq \ControllerUncertaintyParam_\hstep
\| \phi_\hstep(\cdot,\cdot) \|_{\ControllerCovariance_\hstep^{-1}},
\qquad \mbox{where} \qquad \ControllerUncertaintyParam_\hstep \defeq
\cbonus\Bigg\{ \sqrt{ \dim \log \big(\tfrac{\dim \phase
    \nSamples{1:\phase}}{\delta}\big)} + \sqrt{\lambdareg} \Bigg\}.
\end{align}a
Here $\cbonus > 0$ is a universal constant; we use
$\nSamples{1:\phase} = \sum_{\iindex=1}^{\phase} \nSamples{\iindex}$
to denote the samples used in phases $1$ through $p$; and
$\ControllerCovariance_\hstep$ is a cumulative covariance matrix that
is constructed by the algorithm.

The empirically estimated cumulative covariance
$\ControllerCovariance$ is constructed as follows.  Let
$\pi_{\text{Control}}$ denote the controller used to call \qlearning{}
in phase $\phase-1$, and let $\ControllerPolicy$ denote the greedy
policy used by \sqlearning{} in the same phase.  The cumulative
covariance takes the form
\begin{align}
\ControllerCovariance_\hstep = \underbrace{\vphantom{
    \sum_{\jindex=1}^{\ControllerCounter_\phase}}\sum_{\iindex=1}^{\ControllerTotalCounter}
  \phi_{\iindex\hstep} \phi_{\iindex\hstep}^\top \; + \;
  \lambdareg\Identity}_{\substack{\text{\qlearning}
    \\ \text{covariance} \\ \text{returned in phase $\phase-1$}}} \; +
\; \underbrace{\sum_{\jindex=1}^{\ControllerCounter_\phase}
  \phi_{\jindex\hstep}\phi_{\jindex\hstep}^\top}_{\substack{\text{\sqlearning{}}
    \\ \text{covariance} \\ \text{added in phase $\phase-1$}}}, \qquad
\qquad \text{where } \phi_{\iindex\hstep} \sim \pi_{\text{Control}},
\; \phi_{\jindex\hstep} \sim \ControllerPolicy.
\end{align}
Note that $\ControllerCovariance_\hstep$ is formed by the sum of of
the cumulative covariance matrix returned by \qlearning{} along with
additional terms computed by \sqlearning{} between
\cref{line:NewPolicyRepeat} to \cref{line:NewPolicyTrigger} in phase
$\phase-1$.

Notice that the empirical covariance $\ControllerCovariance_\hstep$ is
estimated de novo within each phase $\phase$, and due to statistical
fluctuations, it does not grow monotonically across phases.
Nonetheless, the covariance and the bonus are two devices to measure
the progress of the algorithm.

\paragraph{Information gain and effective dimension:}
We now define the \emph{effective dimension} $\InformationGain_\hstep$
at time step $\hstep$. It is a scalar quantity that governs the
complexity of the exploration problem, defined as
\begin{align}
\InformationGain_\hstep = \max_{\pi} \log \Big(\det \big(\Identity +
\tfrac{\numobs}{\lambdareg} \, \E_{\phi_\hstep \sim \pi}
     [\phi_\hstep\phi_\hstep^\top] \big) \Big).
\end{align}
We note that this notion has been exploited in past
work~\cite{srinivas2009gaussian,yang2020reinforcement,
  agarwal2020pc,du2021bilinear}.  The information gain can be much
smaller than the dimensionality $\dim$ of the feature vectors $\phi$
---that is, we can have $\InformationGain_\hstep \ll \dim_\hstep$.
This scaling holds, for instance, when the feature moment matrix
$\E_{\phi_\hstep\sim \pi} [\phi_\hstep\phi_\hstep^\top]$ is mostly
concentrated along few directions; see \cref{LemInfoGain} in
Appendix~\ref{SecBoundsInfoGain} for details.


\subsection{Guarantees for \sqlearning{}}
\label{SecS4Guarantees}

We are now ready to present our main result, namely a bound on the
regret incurred by all the policies that generate rollouts, including
those played by the \qlearning{} subroutine when it is called by
\sqlearning{}.  We assume that the bonus function is defined according
to \cref{eqn:Bonus} with an appropriately chosen universal constant.

\begin{theorem}[Performance Bound of \sqlearning{}]
\label{thm:Sequoia}
Consider an MDP that is $\terr$-misspecified w.r.t. Bellman closure
(cf. \cref{def:ModelMisspecificationLR}).  Then for any number of
episodes, there exists an event $\Event_\nEpisodes$ that holds with
probability at least $1-\delta$, and under this event, we have the
following guarantees:
\begin{enumerate}
\item[(a)] The average regret of \sqlearning{} is upper bounded as
\begin{align}
\label{EqnRegretBound}  
\AveRegret(\nEpisodes) & \leq c L \Bigg \{
\frac{\horizon}{\sqrt{\nEpisodes}} \; \Bigg[ \big( \sumh
  \InformationGain_\hstep \big) \times \big( \sumh \big( \dim_\hstep +
  1 \big) \InformationGain_\hstep \big) \Bigg]^{1/2} + \terr \Bigg \}
\qquad \mbox{where $L \defeq \log
  \big(\frac{\dim\nPhases\nEpisodes}{\delta} \big)$.}
\end{align}
\item[(b)] The memory complexity is bounded as
  $\order(L\dim^2\horizon\sum_\hstep \InformationGain_\hstep) =
  \order(L\dim^3\horizon^2)$ while the per-step computational
  complexity is $\order(|\ActionSpace|\dim^2)$.
\end{enumerate}
\end{theorem}
\noindent See Appendix~\ref{SecProofSequoia} for the proof of this
claim.

\medskip

\noindent In absence of model misspecification ($\terr = 0$) and the
special case $\dim_\hstep = \InformationGain_\hstep$ for all $\hstep$,
the worst-case regret bound becomes $\ordertil \big( \horizon^{2}
\sqrt{\dim^3 \nEpisodes} \big)$. Here $\ordertil$ includes constant
and logarithmic factors.


\medskip
\myparagraph{Some comments} When the value function is rescaled to be
in $[0, \horizon]$ instead of the unit interval $[0, 1]$, the regret
bound of \sqlearning{} becomes
$\ordertil(\horizon^{3}\sqrt{\dim^3\nEpisodes})$, which is larger by
only a factor of $\horizon$ relative to the state-of-the-art
$\ordertil(\horizon^{2}\sqrt{\dim^3\nEpisodes})$ bound available for
computationally tractable algorithms~\cite{jin2020provably}.

This slightly sub-optimal sample complexity is counter-balanced by a
number of advantages of \sqlearning{}.  One major benefit is the low
memory footprint, which depends only on $\nEpisodes$ via a logarithmic
factor.  To the best of our knowledge, all previous methods that apply
in this setting need to store the full experience, leading to a memory
requirement scaling at least linearly in the number of episodes
$\nEpisodes$.  For problems with a large number of interactions,
this linear scaling can be prohibitive.

Additionally if the horizon is not very large, the \sqlearning{} bound
might be substantially tighter than the guarantees in Jin et
al.~\cite{jin2020provably}, since it primarily scales with the
effective dimension $\InformationGain$.  To the best of our knowledge,
our result gives the first adaptive and tractable algorithm for this
setting with regret bounds depending on the effective dimension
$\InformationGain$.  In contrast, existing algorithms with regret
bounds in terms of the effective dimension need to know its value to
inherit an improved regret
bound~\cite{agarwal2020pc,yang2020reinforcement}; in this sense, they
are non-adaptive guarantees.

Finally, our approximation error guarantees are new for value-based
exploration algorithms and close some of the gaps with respect to
policy gradient methods.  The approximation error of \sqlearning{}
scales with the worst-case off-policy expected prediction error.  To
the best of our knowledge, temporal difference methods that perform
exploration have only been analyzed with $\ell_\infty$-approximation
guarantees.  Instead, in policy gradient
algorithms~\cite{agarwal2020pc,zanette2021cautiously}, the
approximation error is measured in expectation with respect to an
arbitrary comparator.  Our approximation error depends on the
worst-case (with respect to the policies) approximation error, but the
error is still measured in expectation, and furthermore, the
expectation is inside the absolute value.


\subsection{Guarantees for \qlearning{}}
\label{SecS3Guarantees}

When a stationary controller is available, the \qlearning{} algorithm
can be used to maintain a running estimate of the optimal action value
function $\Qstar$ with a low memory footprint.  This guarantee is of
independent interest: the basic protocol illustrated in
\cref{alg:S3Qlearning} serves as a building block for sophisticated
algorithms, with the \sqlearning{} procedure analyzed in this paper
being one example.  Of course, in the controlled setting the quality
of the value function estimate will also depend on how exploratory the
controller policy is; some notation will be introduced shortly to
quantify this.

Let us define the \emph{expected cumulative covariance matrix} after
$\numobs$ samples from the (controller) policy $\pi$ as
\begin{align}
\label{eqn:Qcov}
\QExpectedCovariance{\hstep}{\numobs} = \numobs \; \Big \{
\E_{\phi_\hstep \sim \pi} \big[ \phi_\hstep \phi_\hstep^\top \big] +
\lambdareg \Identity \Big \},
\end{align}
where $\lambdareg > 0$ is a fixed positive regularization
parameter. Given a stationary controller $\pi$ and a bonus function
$\bonus$, the \qlearning{} algorithm returns a sequence of estimated
$Q$-functions $\Qbest{} = (\Qbest{1}, \ldots, \Qbest{\horizon}) \in
\QIBE$.  In this section, we state some bounds on the value function
error $\Qbest{} - \Qstar$ when the bonus function $\bonus = 0$ and the
violation of \fullref{asm:InherentBellman} is measured according to
\cref{def:ModelMisspecificationIBE}. (We consider the extension to the
setting $\bonus \geq 0$ in~\cref{sec:Qexp}).

Our analysis involves the Bellman error associated with $\Qbest{}$,
defined for each $\hstep \in [\horizon]$ and $(\state, \action)$ as
the discrepancy of $\Qbest{}$ in satisfying the Bellman equations with
$\Qbest{\hstep+1}$ as the next state value function:
\begin{align}
\label{EqnDefnErr}
\err_{\hstep}(\state,\action) & \defeq \Qbest{\hstep}(\state,\action)
- \big( \T_\hstep \Qbest{\hstep+1} \big)(\state,\action).
\end{align}
Our analysis shows how this Bellman error can be bounded in terms of
an uncertainty function and an approximation error.  For a given
integer sample size $\numobs > 1$ and user-defined error probability
$\delta \in (0, 1)$, define the scalar quantity
\begin{align}
\label{EqnDefnAlpha}
\QUncertaintyParam{\hstep}{\numobs, \delta} = c \left \{ \sqrt{ \dim
  \log \left( \tfrac{\dim \numobs \epoch \horizon}{\delta} \right)} +
\sqrt{\lambdareg} \right \},
\end{align}
where $c > 0$ is a universal constant, whose specific value can be
determined via the proof.

Suppose that the algorithm terminates after $\nEpochs$ epochs, and let
$\nEpisodes$ be the total number of trajectories.  For a given
tolerance probability $\deltamaster \in (0,1)$, we define the
uncertainty function
\begin{subequations}
\begin{align}
\label{EqnDefnUncertainty}    
\Uerr_\hstep(\state, \action) & \defeq
\QUncertaintyParam{\hstep}{\numobs^*, \delta^*} \|
\phi_\hstep(\state,\action)\|_{\left(\QExpectedCovariance{\hstep}{\numobs^*}
  \right)^{-1}} \qquad \mbox{where $\numobs^* \defeq
  \frac{\nEpisodes}{4\horizon}$, and $\delta^* \defeq
  \frac{\deltamaster}{2\horizon\nEpochs^2 \dim}$.}
\end{align}
We define the comparator error
\begin{align}
\label{EqnDefnComparator}  
\Comparator{\pi}{\hstep}(\state, \action) & \defeq
(\T_\hstep\Qtarget{\hstep + 1})(\state, \action)-
\smlinprod{\phi_\hstep(\state,
  \action)}{\BestPredictor{\pi}{\Qtarget{\hstep + 1}}{\hstep}}.
\end{align}
\end{subequations}

In order to state the theorem, let $\E_{\pi}$ denote the expectation
over the trajectories $(\MyState_1, \MyAction_1,\dots,
\MyState_\horizon, \MyAction_\horizon)$ induced by following $\pi$
after sampling from the starting distribution $\startdistribution$.

\begin{theorem}[Performance bound for \qlearning{}]
\label{thm:Q-learning}
Consider an MDP that is $\terr$-misspecified w.r.t. Bellman closure
(cf. \cref{def:ModelMisspecificationIBE}).  If the $\qlearning{}$
algorithm is run with the uncertainty
function~\eqref{EqnDefnUncertainty}, then for any episode
$\nEpisodes$, with probability at least $1 - \deltamaster$, it returns
a solution $\Qbest{}$ such that:
\begin{enumerate}
\item[(a)] Its Bellman error function $\err_\hstep \defeq
  \Qbest{\hstep} - \big( \T_\hstep \Qbest{\hstep+1} \big)$ satisfies
  the pointwise bound
\begin{subequations}
\begin{align}
\label{eqn:ErrorBracketIBE}
\biggr| \err_{\hstep}(\state, \action) +
\Comparator{\pi}{\hstep}(\state, \action) \biggr| & \leq
\Uerr_{\hstep}(\state, \action) \qquad \mbox{for each $(\state,
  \action)$.}
\end{align}
\item[(b)] For each $\hstep \in [\horizon]$, the greedy policy
  $\pibar_\hstep(\state) \defeq \argmax_{\action}
  \Qbest{\hstep}(\state,\action)$ satisfies the bounds
  \begin{align}
 \label{eqn:VBracketIBE}
- \E_{\pistar} \sumh \Uerr_{\hstep}(\MyState_\taustep,
\MyAction_\hstep) - \terr \leq \E_{\MyState \sim \startdistribution}
(\Vbest{1} - \Vstar_1)(\MyState) \leq \E_{\pibar} \sumh
\Uerr_{\hstep}(\MyState_\taustep, \MyAction_\taustep) + \terr.
  \end{align}
  \end{subequations}
\end{enumerate}
Furthermore, the memory complexity of the algorithm is bounded as
$\order(\dim^2\horizon)$ and the per-step computational complexity is
bounded as $\order(\dim^2 + |\ActionSpace|\dim)$.
\end{theorem}
\noindent See Appendix~\ref{sec:ProofOfS3} for the proof of this
claim.


\section{Discussion}
\label{SecDiscussion}

In this paper, we have introduced several modifications to the basic
$Q$-learning protocol so as to derive an exploratory procedure that
operates with linear function approximation, and is equipped with
performance guarantees while remaining computation and
memory-efficient.  It is natural to ask to what extent these
modifications---the second-order update rule, the use of target
networks and policy replay---are needed in order to obtain such
guarantees.

\medskip

\myparagraph{On the second-order rule} Of the three ingredients, the
second-order update rule only serves to improve the statistical
efficiency.  When the target network and sampling distribution are
fixed, using first-order update rule in~\cref{eqn:BasicQupdate} would
be essentially equivalen to a stochastic gradient update; such updates
would minimizes the loss function~\eqref{eqn:Qloss} with a rate
$1/\sqrt{n}$ instead of the $1/n$ enabled by our second-order updates.
It should be observed that the higher cost and memory of the
second-order rule are not a problem when conducting exploration,
because the main computational bottleneck is the calculation of the
bonus function, while the main memory requirement is due to the policy
replay memory.

Nonetheless, if one is interested purely in the optimization setting,
where neither the replay memory nor the exploration bonuses are
needed, some techniques from variance reduction can be used in
conjunction with a first-order update
rule~\cite{frostig2015competing,li2020root,wainwright2019variance,MouKhaWaiBarJor22}
to lower the computational and space complexity while retaining high
sample efficiency.  We leave this as an interesting direction for
future work.

\medskip

\myparagraph{On the use of target networks} The use of target networks
considerably simplifies the analysis of the algorithm by establishing
a connection with linear regression with a fixed target.  There is no
real downside with adopting target networks, and whether they could be
removed is left as future work.

\medskip

\myparagraph{On the policy replay mechanism} A truly critical
ingredient in this work is the policy replay mechanism, which ensures
that the $Q$-learning updates are performed on data generated by a
stationary controller and with the most recent bonus and value
function estimate.  Without the replay mechanism, the network weights
would be updated with a changing target---recall that the target
networks need to be updated periodically---and under a non-stationary
distribution.  The main issue is that the target networks contain both
statistical errors and bias due to the exploration bonus which decay
with time.  In this case, the high errors present in the target
networks in early phases would hurt the algorithm in later phases.
Discounting early updates by an appropriate learning rate to favor
later updates may seem like a solution, but this can lead to
``catastrophic forgetting'' of past experience because the learning
distribution is non-stationary.  When using linear function
approximation, a concrete consequence is that the algorithm can forget
what it has learned along the directions it played in the early
phases.

Exploration algorithms inspired by Least Square Value Iteration (LSVI)
avoid these issues, because they update the weights of the network by
computing the full least-squares solution using the most recent, and
thus most accurate, target function.  In this way, the next state
value function is as accurate as possible over the full domain, and is
perturbed everywhere by the most recent (and smallest) exploration
bonus, one that truly reflects the current model uncertainty.

Likewise, experience replay would alleviate these issues by
re-training the network using the most recent target network.
Unfortunately, doing so seems to require a replay buffer of size
proportional to all the experience collected so far, making
$Q$-learning no longer a truly streaming algorithm.  The policy replay
memory is a simple solution to such issues, one that preserves the
streaming nature of $Q$-learning.

\medskip

\myparagraph{Future directions} Our work focused exclusively on linear
approximations of $Q$-functions, but some of the underlying ideas are
more generally applicable.  One interesting direction is to extend our
analysis to models with low Eluder
dimension~\cite{wang2020provably,kong2021online}, and to see whether
the regret bound can be improved.  Second, our definition of
approximation error is very permissive, in that it only measures the
expected prediction error.  It would be interesting to understand
whether or not there exist exploration algorithms based on
least-square value iteration (without ``policy replay'') that inherit
similar guarantees.  Finally, this work establishes a partial form of
instance-dependence, in that the results depend on the effective
dimension.  In the simpler tabular setting, the instance dependence of
$Q$-learning has been studied through the lens of local minimax
theory~\cite{KhaXiaWaiJor21, XiaKhaWaiJor22} to obtain completely
sharp instance-dependent guarantees. It would be interesting to
develop similarly sharp analyses in this more general setting with
function approximation.

\section*{Acknowledgment}
This work was partially supported by National Science Foundation FODSI
grant 2023505, DOD ONR Office of Naval Research N00014-21-1-2842,
National Science Foundation DMS grant 2015454, and National Science
Foundation CCF grant 1955450.
Part of this work was completed while Andrea Zanette was visiting
the program Learning and Games at the Simons Institute for the Theory of Computing.
The authors are very grateful to the reviewers, as well as to the meta-reviewer,
for identifying clarity issues present in an earlier draft.


{\small{
\bibliographystyle{plainnat}
\bibliography{rl,Martin}
\medskip
}}

\appendix
\tableofcontents


\section{Proofs for the \qlearning{} algorithm}
\label{sec:ProofOfS3}

This section is devoted to proving the bounds on \qlearning{} stated
in Theorem~\ref{thm:Q-learning}.  At the same time, we also establish
a related result, to be stated momentarily as
Theorem~\ref{thm:Q-learningExp}.  The proofs of both results share a
very similar structure, following the same argument except for the way
in which the violation of \cref{asm:InherentBellman} or
\cref{asm:LowRank} is measured.  More precisely, \cref{thm:Q-learning}
measures misspecification according to $\terr$, which is zero when
\cref{asm:InherentBellman} holds and the bonus is zero.  On the other
hand, \cref{thm:Q-learningExp} measures misspecification using $\terr$
defined according to \cref{def:ModelMisspecificationLR}.  This
quantity is zero under the low-rank assumption
(Assumption~\ref{asm:LowRank}).  Thus, both theorems can be proved
within a common framework, with the only difference being the way in
which $\terr$ is defined.

\medskip

\noindent Let us now state the second result to be proved in this
section.
\begin{theorem}
\label{thm:Q-learningExp}
Consider an MDP that is $\terr$-misspecified w.r.t. the low rank
assumption according to \cref{def:ModelMisspecificationLR}; assume
$\bonus\geq 0$ pointwise.  With the uncertainty
function~\eqref{EqnDefnUncertainty}, for any episode $\nEpisodes$, the
$\qlearning{}$ algorithm returns a solution $\Qbest{}$ with Bellman
error $\err$ such that with probability at least $1-\deltamaster$:
\begin{enumerate}
\item[(a)] The Bellman error function satisfies the pointwise bound
\begin{align}
\label{eqn:ErrorBracketLR}
\min\{ 0, -\Uerr_\hstep + \bonus_\hstep\}(\state, \action) \leq
\big(\err_{\hstep} + \Comparator{\pi}{\hstep} \big)(\state, \action)
\stackrel{}{\leq} \big(\Uerr_{\hstep} + \bonus_\hstep
\big)(\state,\action).
\end{align}
\item[(b)] For each $\hstep \in [\horizon]$, the greedy policy
$\pibar_\hstep(\state) \defeq \argmax_{\action}
\Qbest{\hstep}(\state,\action)$ satisfies the bound
\begin{align}\label{eqn:VBracketLR} 
\E_{\MyState \sim \startdistribution} (\Vbest{1} - \Vstar_1)(\MyState)
\leq \E_{\MyState \sim \startdistribution} (\Vbest{1} - V^{\pibar}_1)(\MyState)
\leq \E_{\pibar} \sumh (\Uerr_{\hstep}+\bonus_\hstep)(\MyState_\taustep,
\MyAction_\taustep) + \terr.
\end{align}
\end{enumerate}
\end{theorem}

\medskip

\myparagraph{Proof } For each epoch, the argument can be broken into
four steps.
\begin{itemize}
\item[Step 1:] First, we show that for any level
$\level \in [\horizon]$, the second-order update
rule~\eqref{line:UpdateRule} produces the same iterates as
least-squares regression would.
\item[Step 2:] Second, we show for any level
$\level \in [\horizon]$, learned predictor $\Qbest{}$ uses at least
$\sim \frac{1}{\horizon}$ of the total data.
\item[Step 3:] Third, we bound the least-square prediction error
under either \cref{asm:InherentBellman} or \cref{asm:LowRank}.  This
analysis controls the error made by the algorithm at each level
$\level \in [\horizon]$ during the epoch under consideration.
\item[Step 4:] The final step is to compute how the
least-square errors propagate and accumulate through timesteps,
thereby leading to the final performance bound in terms of the learned
action value function $\Qbest{}$.
\end{itemize}

\myparagraph{Notation} Let us summarize here some notation for
convenient reference.  We say that the algorithm \emph{has completed
learning} at level $\level$ in epoch $\epoch$ if $\nSamplesLevel =
2^e$, i.e., when the loop over $\nSamplesLevel$ has terminated.  We
indicate with $\nSamplesEpoch = 2^\epoch$ the number of samples
allocated in the epoch $\epoch$.

Let $\{ (\state_\iindex, \action_\iindex, \reward_\iindex,
\stateprime_\iindex) \}_{\iindex=1}^{\nSamplesEpoch}$ be the samples
acquired at level $\level$ in epoch $\epoch$.  For a parameter vector
$\theta$ and a next-state action-value function $Q'$, define the
$(\level, \epoch)$-empirical loss as
\begin{align}
\EmpiricalLoss{\level \epoch}{\theta}{Q'} & \defeq
\sum_{\iindex=1}^{\nSamplesEpoch} \Big [
  \smlinprod{\phi_\level(\state_\iindex,\action_\iindex)}{\theta} -
  r_\iindex - \max_{\actionprime} Q'(\stateprime_{\iindex},
  \actionprime) \Big]^2 + \lambdareg \| \theta \|_2^2,
\end{align}
where $\lambdareg > 0$ is a given regularization parameter.  With a
minor overload of notation, Recalling the class of linear functions
$\QIBE$ from equation~\eqref{eqn:FunctionalSpace}, we define
\begin{align}
Q_{\min} = \argmin_{Q \in \QIBE} \EmpiricalLoss{\level\epoch}{Q}{Q'}
\end{align}
if $Q_{\min}$ can be written as $Q_{\min}:
(\state, \action) \mapsto \smlinprod{\phi(\state, \action)}{\theta_{\min}}$,
where $\theta_{\min} = \arg \min_{\|\theta\|_2 \leq
1} \EmpiricalLoss{\level \epoch}{\theta}{Q'}$.


\subsection{Main argument}

We now proceed to the core of the argument.  When the algorithm
terminates at the evaluation episode $\nEpisodes$, it returns the
predictor $\Qbest{}$. For the rest of the proof, we let $\nEpochs$ be
the epoch in which $\Qbest{}$ was last updated upon termination of the
algorithm.  Moreover, our proof makes use of three auxiliary lemmas,
which we begin by stating.

\begin{lemma}[Equivalence with Least-Squares]
\label{lem:EquivalenceWithLeastSquares}
Upon completion of level $\level$ within epoch $\epoch$, the
\qlearning{} algorithm returns a parameter vector $\QparTarget_{\level
  \epoch}$ such that
\begin{align}
\QparTarget_{\level\epoch} & = \arg \min_{\| \theta \|_2 \leq 1} \left
\{ \EmpiricalLoss{\level\epoch}{\theta}{\Qtarget{\level+1,\epoch}}
\right \}.
\end{align}
\end{lemma}
\noindent See \cref{sec:ProofEquivalenceWithLeastSquares} for the
proof of this claim.

\medskip

\noindent We emphasize that the target network
$\Qtarget{\level+1,\epoch}$ remains fixed throughout a given epoch.
The lemma establishes that the second-order update rule produces the
same solution as a batch least-square regression would.

\medskip
\noindent Our next step is to lower bound the number of samples used
to solve the regression problem:
\begin{lemma}[Number of Samples]
\label{lem:OrderOfUpdates}
Upon termination in episode $\nEpisodes$, the algorithm returns a
parameter sequence $\QparBest{} = \{ \QparBest_\level
\}_{\level=1}^\horizon$ such that
\begin{align}
\label{eqn:QparBest}
\QparBest_{\level} = \argmin_{\| \theta\|_2 \leq 1}
\EmpiricalLoss{\level,\nEpochs}{\theta}{\Qbest{\level+1}} \qquad
\mbox{for each $\level \in [\horizon]$}
\end{align}
and moreover, the level $\ell$ regression problem uses at least
$\nSamplesLevel \geq \frac{\episodetot}{4\horizon}$ of the form
$\{(\state_\iindex, \action_\iindex, \reward_\iindex,
\stateprime_\iindex)\}_{\iindex=1}^{\nSamplesLevel}$.
\end{lemma}
\noindent The above lemma states that nearly all the data collected
are used. The proof can be found in~\cref{sec:LearningMechanics}.

\medskip

Given this equivalence to least-squares regression and the lower bound
on the sample size, we can now leverage standard analysis of linear
regression so as to bound the prediction error.  Recall our definition
of the error function
\begin{align}
\err_{\level \epoch}(\state, \action) & \defeq
\Qtarget{\level\epoch}(\state,\action) - \left( \T_\level
\Qtarget{\level+1,\epoch} \right)(\state,\action).
\end{align}

\begin{lemma}[Least Square Error Bounds]
\label{lem:LeastSquareErrorBounds}
The \qlearning{} procedure returns a predictor $\Qbest{\level}$ whose
error $\err_{\level}$ is sandwiched as
\begin{align}
\label{eqn:SymmetricBound}
\min \Big \{ 0, \left(-\Uerr_\level(\state, \action) +
\bonus_{\level}(\state, \action) \right) \Big \} -
\Comparator{\pi}{\level}(\state,\action) \stackrel{(a)}{\leq}
\err_{\level}(\state,\action) \stackrel{(b)}{\leq}
\Uerr_\level(\state, \action) + \bonus_{\level}(\state, \action)
-\Comparator{\pi}{\level}(\state,\action)
\end{align}
with probability at least $1 - \delta$.
\end{lemma}
\noindent See \cref{sec:proofOfLeastSquareErrorBound} for the proof of
this claim.

\medskip

\cref{lem:LeastSquareErrorBounds} allows
us to quantify the empirical Bellman backup error and
uses \fullref{asm:InherentBellman} or \fullref{asm:LowRank} depending
on the setting (i.e., optimization vs exploration).  In the zero-bonus
setting ($\bonus = 0$), the error term can be bounded symmetrically by
the uncertainty function (see \cref{eqn:SymmetricBound}), which is
always positive.  The $\min$ function on the left hand side arises due
to ``clipping'' the $\Qtarget{\level}$ values, so that adding a bigger
bonus $\bonus_\level$ does not necessarily make $\Qtarget{\level}$
(and its error function) more positive.

Our next step to establishing the bounds~\eqref{eqn:VBracketIBE} is to
analyze how errors propagate.  We begin with the rightmost inequality
in equation~\eqref{eqn:VBracketIBE}.  Since $\Qstar$ is the optimal
$Q$-function, we have the pointwise inequality
\begin{align*}
(\Qbest{} - \Qstar)(\state, \action)
\leq (\Qbest{} - Q^{\pibar})(\state, \action) \quad \mbox{for all $(\state, \action)$ pairs}, 
\end{align*}
valid for any policy $\pibar$; in particular, this bound holds for the
greedy policy $\pibar$ with respect to $\Qbest{}$.  Moreover, for this
greedy policy, we have
\begin{align*}
\Qbest{\hstep}& = \err_\hstep+ \left(\T_\hstep \Qbest{\hstep+1}
\right) = \err_\hstep + \left(\T^{\pibar}_\hstep \Qbest{\hstep+1} \right),
\quad \quad \quad \mbox{for all $\hstep \in [\horizon]$.}
\end{align*}
Since $Q^{\pibar}$ satisfies the Bellman evaluation equations
$Q^{\pibar}_{\hstep} = \left(\T^{\pibar}_\hstep Q^{\pibar}_{\hstep+1}
\right)$ for each $\hstep \in [\horizon]$, the claim now follows, as
$\Qbest{\hstep}$ can be thought of as the action value function of
$\pibar$ on an MDP with dynamics specified by $\T^{\pibar}$, and
reward function consisting of a portion from $\T^{\pibar}$, along with
an additional reward equal to $\err$.

The proof of the left inequality in equation~\eqref{eqn:VBracketIBE}
is similar.  In particular, we observe that
\begin{align*}
\Qbest{\hstep} & = \err_\hstep + \left(\T_\hstep \Qbest{\hstep+1}
\right) \geq \err_\hstep + \left(\T^{\pistar}_\hstep \Qbest{\hstep+1}
\right), \quad \quad \quad \mbox{for all $\hstep \in [\horizon]$.}
\end{align*}
Expanding the definition of error function along the trajectories
identified by $\pistar$ concludes the proof of the
claim.


\subsection{Proof of \texorpdfstring{\fullref{lem:EquivalenceWithLeastSquares}}{}}
\label{sec:ProofEquivalenceWithLeastSquares}

Fix an epoch $\epoch$ and let the current level be $\level$. By
construction, the target network for the next state value function
$\Qtarget{\level+1,\epoch}$ has already been updated in that epoch.
We observe that $\qlearning$ is updating $\Qpar_\hstep$ in a way
equivalent to \cref{alg:ShermanMorrisonUpdate} with $a_\kindex=
\phi_h$ and $b_\kindex = \reward_h + \max_{\actionprime}
\Qtarget{\hstep+1}(\state_{\hstep+1},\actionprime)$ and $\B = \{x \mid
\| x \|_2 \leq 1 \}$.

\begin{algorithm}[H]
\caption{\textscTemporaryFix{Streaming Least Squares}}
\label{alg:ShermanMorrisonUpdate}
\begin{algorithmic}[1]
\State $\Sigmavarsecond_1 = \lambdareg \Identity_{d\times d}$ \State
$\xvarsecond_0 = 0$ \For{$\kindex=1,2,\dots$} \State Receive
$(\avar_\kindex,\bvar_\kindex)$ \State $\xvarsecond_k =
\xvarsecond_{\kindex-1} +
\frac{\Sigmavar_{\kindex}^{-1}\avar_{\kindex}\left(\bvar_\kindex -
  \avar_\kindex^\top\xvarsecond_{\kindex-1} \right)}{1+
  \avar_\kindex^\top \Sigmavar_{\kindex}^{-1}
  \avar_\kindex}$ \label{line:x} \State
$\Sigmavarsecond_{\kindex+1}^{-1} = \Sigmavarsecond_{\kindex}^{-1} -
\frac{\Sigmavarsecond_{\kindex}^{-1} \avar_\kindex\avar^\top_\kindex
  \Sigmavarsecond_{\kindex}^{-1} }{1+ \avar_\kindex^\top
  \Sigmavarsecond_{\kindex}^{-1}
  \avar_\kindex}$ \label{line:Covupdate} \EndFor \\ \Return
$\argmin_{\|\xvar\|_2 \leq 1} \| \xvar-\xvarsecond_\Kindex
\|^2_{\Sigmavarsecond_{\Kindex+1}}$
\end{algorithmic}
\end{algorithm}

Thus, in order to prove \cref{lem:EquivalenceWithLeastSquares}, it
suffices to show that \cref{alg:ShermanMorrisonUpdate} finds the
empirical risk minimizer.  In particular, we claim that given any
sequence of tuples
$(\avar_1,\bvar_1),\dots,(\avar_\kindex,\bvar_\kindex)$, then upon
termination, \cref{alg:ShermanMorrisonUpdate} returns the constrained
minimizer
\begin{align}
  \label{eqn:ConstrainedERM}
  \arg \min_{\|\xvar\|_2 \leq 1} \Big \{ \sumk \big(\xvar^\top
  \avar_\kindex - \bvar_\kindex \big)^2 + \lambdareg \| x \|^2_2 \Big
  \}.
\end{align}

In order to prove this claim, we introduce some helpful notation.
With the initialization \mbox{$\Avar_0 \defeq \sqrt{\lambdareg}
  \Identity_{d\times d}$} and \mbox{$\Bvar_0 \defeq 0$,} define the
recursions
\begin{align}
\Avar_\kindex \defeq \begin{bmatrix}
  \Avar_{\kindex-1}\\ \avar_\kindex^\top
\end{bmatrix}, \qquad \Bvar_\kindex \defeq \begin{bmatrix}
\Bvar_{\kindex-1}\\
\bvar_\kindex
\end{bmatrix}
\end{align}
The associated solution to the normal equations is given by
\mbox{$\xvarprime_{\kindex} \defeq
  (\Avar_{\kindex}^\top\Avar_{\kindex})^{-1}\Avar_{\kindex}^\top
  \Bvar_\kindex$.}  With these definition, we then have the
equivalences
\begin{align}
\xvarprime_\kindex \overset{\text{(i)}}{=} \argmin_{\xvar} \|
\Avar_\kindex \xvar - \Bvar_\kindex\|_2^2 \overset{\text{(ii)}}{=}
\sum_{\iindex=1}^{\kindex} \left(\xvar^\top \avar_\iindex -
\bvar_\iindex \right)^2 + \lambdareg \| x \|^2_2,
\end{align}
where step (i) follows by definition; and step (ii) follows from how
the dataset $(\Avar_0, \bvar_0)$ was constructed, in particular
including the pair
$\{(\sqrt{\lambdareg}e_1,0),\dots,(\sqrt{\lambdareg}e_d,0)\}$ prior to
the $\kindex$ samples
$\{(\avar_\iindex,\bvar_\iindex)\}_{\iindex=1}^\kindex$.

Now we proceed by induction on the index $\kindex$.  For the base case
$\kindex = 0$, observe that $\xvarprime_0 = \xvarsecond_0$ and
$\Avar_{0}^\top \Avar_{0} = \lambdareg \Identity$ by the given
initialization.  Turning to the induction step, let us suppose that
the equalities
\begin{align}
\label{eqn:InductiveHypothesis} 
\xvarprime_{\kindex-1} = \xvarsecond_{\kindex-1}, \quad \mbox{and}
\quad \Avar_{\kindex-1}^\top\Avar_{\kindex-1} & =
\Sigmavarsecond_{\kindex}
\end{align}
hold for a certain $\kindex \geq 1$. We can write the next iterate
$\xvarprime_{\kindex}$ as
\begin{align*}
\xvarprime_{\kindex} =
(\Avar_{\kindex}^\top\Avar_{\kindex})^{-1}\Avar_{\kindex}^\top
\Bvar_\kindex & \overset{}{=} \left(\Avar_{\kindex-1}^\top
\Avar_{\kindex-1} +
\avar_\kindex\avar^\top_\kindex\right)^{-1} \begin{bmatrix}
  \Avar^\top_{\kindex-1} & \avar_\kindex
\end{bmatrix} \begin{bmatrix}
\Bvar_{\kindex-1} \\
\bvar_\kindex	
\end{bmatrix}  \\
& \overset{\text{(iii)}}{=} \Bigg[ \left( \Avar_{\kindex-1}^\top
  \Avar_{\kindex-1}\right)^{-1} - \frac{\left( \Avar_{\kindex-1}^\top
    \Avar_{\kindex-1}\right)^{-1} \avar_\kindex\avar^\top_\kindex
    \left(\Avar_{\kindex-1}^\top \Avar_{\kindex-1}\right)^{-1} }{1+
    \avar_\kindex^\top \left( \Avar_{\kindex-1}^\top
    \Avar_{\kindex-1}\right)^{-1} \avar_\kindex} \Bigg] \left(
\Avar_{\kindex-1}^\top \Bvar_{\kindex-1} +
\avar_{\kindex}\bvar_{\kindex} \right),
\end{align*}
where step (iii) follows from the Sherman Morrison rank-one matrix
inversion formula (e.g.,~\cite{golub2012matrix}). Recalling that
\begin{align*}
\xvarprime_{\kindex-1} = \left(\Avar_{\kindex-1}^\top
\Avar_{\kindex-1}\right)^{-1}\Avar^\top_{\kindex-1}\Bvar_{\kindex-1},
\end{align*}
we find that
\begin{align*}
\xvarprime_{\kindex} & = \xvarprime_{\kindex-1} +
\left(\Avar_{\kindex-1}^\top
\Avar_{\kindex-1}\right)^{-1}\avar_\kindex\bvar_\kindex - \frac{\left(
  \Avar_{\kindex-1}^\top \Avar_{\kindex-1}\right)^{-1}
  \avar_\kindex\avar^\top_\kindex \left(\Avar_{\kindex-1}^\top
  \Avar_{\kindex-1}\right)^{-1} }{1+ \avar_\kindex^\top \left(
  \Avar_{\kindex-1}^\top \Avar_{\kindex-1}\right)^{-1} \avar_\kindex}
\left( \Avar_{\kindex-1}^\top \Bvar_{\kindex-1} +
\avar_{\kindex}\bvar_{\kindex} \right) \\ & = \xvarprime_{\kindex-1} +
\frac{ \left(\Avar_{\kindex-1}^\top
  \Avar_{\kindex-1}\right)^{-1}\avar_\kindex\bvar_\kindex - \left(
  \Avar_{\kindex-1}^\top \Avar_{\kindex-1}\right)^{-1}
  \avar_\kindex\avar^\top_\kindex
  \overbrace{\left(\Avar_{\kindex-1}^\top
    \Avar_{\kindex-1}\right)^{-1} \Avar_{\kindex-1}^\top
    \Bvar_{\kindex-1}}^{\xvarprime_{\kindex-1}} }{1+
  \avar_\kindex^\top \left( \Avar_{\kindex-1}^\top
  \Avar_{\kindex-1}\right)^{-1} \avar_\kindex} \\ & =
\xvarprime_{\kindex-1} + \frac{ \left(\Avar_{\kindex-1}^\top
  \Avar_{\kindex-1}\right)^{-1}\avar_\kindex\left( \bvar_\kindex
  -\avar^\top_\kindex \xvarprime_{\kindex-1} \right)}{1+
  \avar_\kindex^\top \left( \Avar_{\kindex-1}^\top
  \Avar_{\kindex-1}\right)^{-1} \avar_\kindex}.
\end{align*}
Since $\xvarprime_{\kindex-1} = \xvarsecond_{\kindex-1}$ by the
induction hypothesis, it follows that $\xvarprime_{\kindex} =
\xvarsecond_{\kindex}$ as the above display matches \cref{line:x} of
\cref{alg:ShermanMorrisonUpdate}.  Applying the Sherman-Morrison rank
one update, we find that
\begin{align}
(\Sigmavarsecond_{\kindex+1})^{-1} & =
  (\Avar_{\kindex}^\top\Avar_{\kindex})^{-1}.
\end{align}
Thus, we have established that the
equalities~\eqref{eqn:InductiveHypothesis} hold for every $\kindex$.
The proof is concluded upon noticing that $\xvarsecond_k$
is the unconstrained solution to the loss in
\eqref{eqn:ConstrainedERM},
and the projection step in the final line of the algorithm
is thus equivalent to solving \eqref{eqn:ConstrainedERM}
with the constraint $\xvar \in \B$.


\subsection{Proof of \texorpdfstring{\fullref{lem:OrderOfUpdates}}{}}
\label{sec:LearningMechanics}

In every epoch the algorithm updates the target networks in the order
$\Qtarget{\horizon,\epoch}, \ldots,\Qtarget{1,\epoch}$. Since the
algorithm returns \mbox{$\Qbest{\level} = \Qtarget{\level,\nEpochs}$}
for each $\level \in [\horizon]$, the statement~\eqref{eqn:QparBest}
follows by construction of the algorithm
and~\cref{lem:EquivalenceWithLeastSquares}.

It remains to lower bound the number of samples involved in the
computation of $\Qbest{\level}$. By construction, every epoch $\epoch$
uses exactly $2^\epoch \horizon$ trajectories. Let $m$ denote the
number of trajectories in the current (unfinished) epoch $\nEpochs+1$
when we evaluate the algorithm (at the stopping time in the episode
$\nEpisodes$). We must have
\begin{align}
K = \horizon \left( 2^1 + 2^2 + \dots + 2^{\nEpochs} + m \right) \leq
\horizon \left( 2^{\nEpochs+1} + 2^{\nEpochs+1} \right) = 4\horizon
2^{\nEpochs} .
\end{align} 
Since the algorithm in epoch $\nEpochs$ has sampled $\horizon
2^{\nEpochs}$ trajectories, using the above relation we deduce that it
must have used $\horizon 2^{\nEpochs} \geq \frac{\episodetot}{4}$
total episodes, meaning at least $\nSamplesLevel \geq
\frac{\episodetot}{4\horizon}$ in every level to solve the regression
problem~\eqref{eqn:QparBest}, as stated.


\subsection{Proof of \texorpdfstring{\fullref{lem:LeastSquareErrorBounds}}{}}
\label{sec:proofOfLeastSquareErrorBound}
Let $\{ \left( \state_\iindex, \action_\iindex, \reward_\iindex,
\stateprime_\iindex \right) \}_{\iindex =1}^{\nSamplesEpoch}$ be
the sequence of $\nSamplesEpoch$ states, actions, rewards, successor
states acquired while learning level $\level$ in epoch $\epoch$.  We
drop the epoch index $\epoch = \nEpochs$ as this is fixed through the
proof.

Within an epoch, \qlearning{} updates the target networks in the order
$\Qtarget{\horizon}, \Qtarget{\horizon-1},
\ldots, \Qtarget{2}, \Qtarget{1}$.  Thus, when the algorithm 
updates $\Qtarget{\level}$, it must use $\Qtarget{\level+1}$, which
has already been updated in that epoch, to compute the backup
in \cref{line:TDError}. Notice that the next-timestep target
$\Qtarget{\level+1}$ stays fixed while learning at level $\level$.
Observe that regardless of the choice of bonus (so in either the
optimization or exploration setting), we
have \mbox{$\| \Qtarget{\level+1}\|_\infty \leq 1$} by construction.

We introduce the
  shorthand \mbox{$\Comparator{\pi}{\iindex} \defeq \left( \T_\level \Qtarget{\level
  + 1} \right) (\state_\iindex, \action_\iindex)
  - \innerprod{\phi_\level(\state_\iindex, \action_\iindex)
  }{\BestPredictor{\pi}{ \Qtarget{\level + 1}}{\level}}$} for the
  comparator error evaluated at $(\state_\iindex, \action_\iindex)$.
With this shorthand, we can write
\begin{align}
r_\iindex + \max_{\actionprime} \Qtarget{\level+1}
(\stateprime_{\iindex}, \actionprime) = \left(\T_\level
\Qtarget{\level + 1} \right)(\state_\iindex, \action_\iindex) +
\noise_\iindex =
\innerprod{\phi_\level(\state_\iindex, \action_\iindex)}
{\BestPredictor{\pi}{ \Qtarget{\level + 1}}{\level}}
+ \Comparator{\pi}{\iindex} + \noise_\iindex,
\end{align}
where $\noise_\iindex \defeq r_\iindex
+ \max_{\actionprime} \Qtarget{\level+1}
(\stateprime_{\iindex}, \actionprime) - \left(\T_\level
\Qtarget{\level+1} \right)(\state_\iindex, \action_\iindex)$ is the
Bellman noise.  Note that conditioned on
$(\state_\iindex,\action_\iindex)$, the random variable on the left
hand side is bounded in $[-1,+1]$.

Now, to conclude we use the following high-probability error bound on
a perturbed least-squares estimator.  Given a joint distribution $\mu$
over pairs $(X, Y)$, define the constrained least-squares estimate
\begin{subequations}
\begin{align}
\label{eqn:PopulationMinimizer}
\theta^\star \defeq \min_{\|\theta\|_2 \leq 1} \E_{(X, Y)}
\left(\smlinprod{X}{\theta} - Y \right)^2.
\end{align}
Given $\numobs$ i.i.d. samples $(x_i, y_i) \sim \mu$, we define the
empirical version of this estimator
\begin{align}
\label{eqn:EmpiricalMinimizer}
\widehat \theta \defeq \min_{\|\theta\|_2 \leq 1} \frac{1}{\numobs}
\sumi \left( \smlinprod{x_i}{\theta} - y_i \right)^2.
\end{align}
\end{subequations}
The following result bounds the difference between the empirical
and population estimates:
\begin{lemma}[Convergence to Population Minimizer]
\label{lem:Convergence2PopulationMinimizer}
The empirical estimate~\eqref{eqn:EmpiricalMinimizer} satisfies the
bound
\begin{align}
\|\widehat \theta - \theta^\star\|_{\big(n\E_{\mu} xx^\top + \lambda\Identity \big)} & \leq c\Big\{ \sqrt{\dim\log
  \frac{\dim n}{\delta} } + \sqrt{\lambda} \Big\}
\end{align}
with probability at least $1 - \delta$.
\end{lemma}
\noindent See \cref{sec:Convergence2PopulationMinimizer} for the proof
of this claim.

\medskip

After redefining $\delta$ and collecting probabilities,
Cauchy-Schwartz now ensures with probability $1 - \delta$ that
\begin{align}
\label{EqnMinaBound}
\big| \smlinprod{\phi_\level(\state,\action)}{\QparTarget_{\level} -
\BestPredictor{\pi}{\QparTarget_{\level+1}}{\level}} \big | & \leq \| \phi_\level(\state,\action)\|_{
\left( \QExpectedCovariance{\level}{n^\star} \right)^{-1}} \|
\QparTarget_{\level} -  \BestPredictor{\pi}{\QparTarget_{\level+1}}{\level}
\|_{\QExpectedCovariance{\level}{n^\star}} \; \leq \;
\underbrace{\ExpectedUncertaintyParamNoPhase{n^\star, \delta^\star}
\| \phi_\level(\state,\action) \|_{\left(
\QExpectedCovariance{\level}{n^\star}\right)^{-1}}}_{\equiv
\Uerr_\level(\state, \action)},
\end{align}
where we have used the definition~\eqref{EqnDefnUncertainty} for the
uncertainty parameter with the number of samples given
by \fullref{lem:OrderOfUpdates} together with a union bound over the
horizon and the random epoch at evaluation time.

\medskip
\noindent We now use the bound~\eqref{EqnMinaBound} to prove the two
inequalities in equation~\eqref{eqn:SymmetricBound}.
\medskip

\myparagraph{Proof of inequality~\eqref{eqn:SymmetricBound}(b)} We
begin with the upper bound.  Due to the clipping step, we are
guaranteed to have the upper bound $\Qtarget{\level}(\state, \action)
\leq \smlinprod{\phi_\level(\state, \action)}{\QparTarget_{\level}} +
\bonus_\level(\state, \action)$, and hence
\begin{align*}
\Qtarget{\level}(\state,\action) - \smlinprod{\phi_\level(\state,
  \action)}{\BestPredictor{\StationaryController}{\QparTarget_{\level+1}}{\level}}
& = \smlinprod{\phi_\level(\state, \action)}{\QparTarget_\level -
  \BestPredictor{\StationaryController}{\QparTarget_{\level+1}}{\level}}
+ \bonus_\level(\state, \action) \; \stackrel{(i)}{\leq} \;
\Uerr_\level(\state,\action) + \bonus_\level(\state,\action),
\end{align*}
where step (i) follows from our earlier
inequality~\eqref{EqnMinaBound}.  Thus, we have established the
bound~\eqref{eqn:SymmetricBound}(b) once we recall the definition of
transfer error $\Comparator{\pi}{\level}$.

\medskip


\myparagraph{Proof of the lower bound~\eqref{eqn:SymmetricBound}(a)}
We now turn to the lower bound.  By construction, we have
$\|\Qtarget{\level}\|_\infty \leq 1$, so that we can write
\begin{align*}
\Qtarget{\level}(\state,\action) = \min \big \{ 1,
\smlinprod{\phi_\level(\state, \action)}{\QparTarget_{\level}} +
\bonus_\level(\state, \action) \big \}.
\end{align*}
Consequently, by adding and subtracting the term
$\smlinprod{\phi_\level(\state, \action)}{\BestPredictor{ \StationaryController}{\QparTarget_{\level+1}}{\level}}$,
we have
\begin{align*}
\Qtarget{\level}(\state, \action) & = \min \Big \{ 1,
\smlinprod{\phi_\level(\state, \action)}{\QparTarget_{\level} -
  \BestPredictor{\StationaryController}{\QparTarget_{\level+1}}{\level}}
+ \bonus_\level(\state,\action) +
\smlinprod{\phi_\level(\state,\action)}{\BestPredictor{
    \StationaryController}{\QparTarget_{\level+1}}{\level}} \Big \} \\
& \stackrel{(i)}{\geq} \min \Big \{ 1, -\Uerr_\level(\state,\action) +
\bonus_\level(\state,\action) +
\smlinprod{\phi_\level(\state,\action)}{\BestPredictor{\StationaryController}{\QparTarget_{\level+1}}{\level}}
\Big \} \\
& \stackrel{(ii)}{\geq} \min \Big \{ 0, - \Uerr_\level(\state,\action)
+ \bonus_\level(\state,\action) \Big \} +
\smlinprod{\phi_\level(\state,
  \action)}{\BestPredictor{\StationaryController}{\QparTarget_{\level+1}}{\level}}.
\end{align*}
where step (i) uses our earlier bound~\eqref{EqnMinaBound}, and step
(ii) follows since
$\smlinprod{\phi_\level(\state,\action)}{ \BestPredictor{\StationaryController}{\QparTarget_{\level+1}}{\level}}
\leq 1$.  In this way, we have shown that
\begin{align*}
\Qtarget{\level}(\state,\action) - (\T_\level
\Qtarget{\level+1})(\state,\action) \geq \min \Big\{ 0, -
\Uerr_\level(\state,\action) + \bonus_\level(\state,\action) \Big\} -
\Comparator{\pi}{\level}(\state,\action),
\end{align*}
as claimed in equation~\eqref{eqn:SymmetricBound}(a).


\subsection{Proof of Lemma~\ref{lem:Convergence2PopulationMinimizer}}
\label{sec:Convergence2PopulationMinimizer}

Our proof makes use of known bounds on the excess risk in a linear
regression problem.  In particular, consider a regression problem
based on covariate vectors $\phi \in \real^d$ and responses $y \in
\real$ that satisfy the bounds $\|\phi\|_2 \leq 1$ and $|y| \leq
y_{max}$.

With the shorthand $z = (\phi, y)$, define the least-squares loss
$\LeastSq_w(z) = \tfrac{1}{2} (y - \smlinprod{\phi}{w} )^2$.  For some
distribution $\mprob$ over $\real^d \times \real$, define the
constrained population and empirical minimizers
\begin{align*}
w^\star \defeq \arg \min_{ \|w\|_2 \leq \BallRad} \E_{Z \sim \mprob}
\LeastSq_w(Z), \quad \mbox{and} \quad \what \defeq \arg \min_{ \|w\|_2
  \leq \BallRad} \frac{1}{\numobs} \sum_{i=1}^\numobs \LeastSq_w(Z_i)
\end{align*}
where $\{Z_i\}_{i=1}^\numobs$ are drawn i.i.d. according to $\mprob$.

We claim that the excess risk associated with the constrained
least-squares estimate $\what$ can be bounded as
\begin{align}
\label{EqnMehtaBound}  
\E_{Z\sim \mprob}[\LeastSq_{\what}(Z) - \LeastSq_{\wstar}(Z)] \leq
\frac{1}{\numobs} \Big \{ 32(\BallRad + y_{max})^2\times \Big[d \log
  \big(32 \BallRad \numobs (\BallRad + y_{max}) \big) + \log \big(
  \tfrac{1}{\delta'} \big)\Big]+1 \Big \},
\end{align}
with probability at least $1 - \delta'$.  This bound follows as a
consequence of a result due to Mehta~\citep{mehta2017fast}.  In
particular, the maximum value the loss can take is $L^2_{max} = (\BallRad +
y_{max})^2$. Applying Theorem 1 in the paper~\citep{mehta2017fast} to
the least-squares objective, which is $1/(4L_{max}^2)$-exp-concave,
yields the claim.

\medskip

\noindent To conclude, the proof of
\cref{lem:Convergence2PopulationMinimizer} follows by combining the
bound with \fullref{lem:Sigma2ERM}.


\section{Proof of Theorem~\ref{thm:Sequoia}}
\label{SecProofSequoia}

This section is devoted to the proof of the performance bound on
\sqlearning{} stated in Theorem~\ref{thm:Sequoia}.  At a high level,
the proof consists of three steps.  First, we decompose the regret
into a sum of partial regrets incurred in each phase. Second, we show
that the exploration bonus correctly quantify the uncertainty; this
allows the algorithm to ensure optimism. Finally, we use an elliptic
potential argument and a bound on the number of phases to conclude the
proof.

\paragraph{Notation:}
Letting $\phase$ be the current phase, we use $\Qfunction{\phase}$ and
$\Vfunction{\phase}$ (respectively) to denote the $Q$-action-value and
value functions returned by \qlearning{}, (cf.
\cref{line:SequoiaCallsQlearning}) and $\InControllerTheta{\phase}$ to
denote the associated parameter of its linear representation. We let
$\Policy{\phase}$ be the policy extracted in \cref{line:GreedyPolicy}
of \sqlearning{} after termination of \qlearning. Let
$\nSamples{\phase}$ be the (random) number of times that the policy is
executed in phase $\phase$ between \cref{line:NewPolicyRepeat} and
\cref{line:NewPolicyTrigger} of \sqlearning{}
(\cref{alg:Sequoia}). Denote with $\nPhases$ the total number of
phases, including the one that is still in progress at the evaluation
episode $\nEpisodes$.  We let $\bonusnew{\phase}$ denote the
bonus~\eqref{line:bonusS4} created at the end of phase $\phase-1$ in
\cref{line:bonusS4} of \cref{alg:Sequoia}; the bonus will be actively
used in phase $\phase$. We use the shorthand
$\Comparator{(\phase)}{\hstep}$ with the same meaning as
\cref{EqnDefnComparator} where the policy $\pi$ is the one used in
phase $\phase$ within \qlearning{}.


\subsection{Main argument}

We begin by decomposing the regret of \sqlearning{} into the sums of
partial regrets generated in each phase $\phase = 1,\dots,\nPhases$.
The partial regret $\Regretphase{\phase}$ in phase $\phase$
corresponds to the regret incurred by playing all policies in that
phase. Policy rollouts that generate regret are performed in one of
two places: (i) in the call to \qlearning{} (see
\cref{line:SequoiaCallsQlearning} of \sqlearning{}), or; (ii) by
\sqlearning{} between \cref{line:NewPolicyRepeat} and
\cref{line:NewPolicyTrigger}. We can write
\begin{align}
\Regretphase{\phase} \defeq \sum_{\substack{\pi \text{ played in}
    \\ \text{the call to} \\ {\qlearning} \\ \text{in phase
      $\phase$}}} \E_{\MyState_1 \sim \startdistribution}\left(
\VstarFH{1} - \VpiagentFH{1}{} \right)(\MyState_{1}) +
\sum_{\substack{\pi \text{ played in} \\ \text{the main loop of}
    \\ {\sqlearning} \\ \text{in phase $\phase$}}}\E_{\MyState_1 \sim
  \startdistribution}\left( \VstarFH{1} - \VpiagentFH{1}{}
\right)(\MyState_{1}).
\end{align}
The following lemma leverages the mechanics of the algorithm to upper
bounds the total regret up to episode $\nEpisodes$ by expressing it as
a sum of the regrets generated only by the main loop of \sqlearning{}
(so excluding the call to \qlearning{}).  What makes this possible is
that \qlearning{} only plays the policies from the \prm{}
$\sqlearningPolicyCover{\phase}$, which have already been played by
the controller (\sqlearning{}).  Summing over the phases and
accounting for possible statistical deviations (due to sampling from
the policy mixture) yields the following claim:
\begin{lemma}[Phased Regret]
\label{lem:PhasedRegret}
With probability at least $1-\delta$, uniformly over all $\nEpisodes$,
we have the upper bound $\Regret(\nEpisodes) \defeq \sump
\Regretphase{\phase} \leq \Term_1 + \Term_2$, where
\begin{subequations}
\begin{align}
\Term_1 & \defeq \horizon \Big \{ \sump \sum_{\jindex=1}^{\phase-1}
\nSamples{\jindex}\times \Estart \big[\Vstar_1(\MyState) -
  \RealV{\jindex}_1(\MyState) \big] \Big \} + \horizon \log \big(
\tfrac{\nPhases}{\delta}\big) \quad \mbox{and} \\
\Term_2 & \defeq \sump \nSamples{\phase}\times \Estart \big[
  \Vstar_1(\MyState) - \RealV{\phase}_1(\MyState) \big].
\end{align}
\end{subequations}
\end{lemma}
\noindent 
See \cref{sec:PhasedRegret} for the proof of this claim.

\medskip
 
\noindent
Note that $\Term_1$ corresponds to the regret associated with
\qlearning{}, whereas $\Term_2$ is associated with \sqlearning{}.  To
be clear, the regret to \sqlearning{} excludes the regret of the
policies that are rolled out within the \qlearning{} subroutine.  Now
notice that for some constant $c \in\R$ we can write the total number
of episodes as
\begin{align*}
\nEpisodes = \underbrace{c\horizon\sump \sum_{\jindex=1}^{\phase-1}
  \nSamples{\jindex}}_{\substack{\text{Total trajectories by}
    \\ \text{\qlearning{}}}} + \underbrace{\sump
  \nSamples{\phase}}_{\substack{\text{Total trajectories by}
    \\ \text{\sqlearning{}}}}.
\end{align*}
Applying the Cauchy--Schwarz inequality yields the bound
$\Regret(\nEpisodes) \leq \Term_3 \times \sqrt{\Term_4}$, where
\begin{align*}
\Term_3 & \defeq \sqrt{c\horizon\sump \sum_{\jindex=1}^{\phase-1}
  \nSamples{\jindex}, \sump \nSamples{\phase}}, \quad \mbox{and}
\\
\Term_4 & \defeq c \horizon\sump \sum_{\jindex=1}^{\phase-1}
\nSamples{\jindex} \Big[\Estart\left(\Vstar_1 - \RealV{\jindex}_1
  \right)(\MyState)\Big]^2 + \sump
\nSamples{\phase}\Big[\Estart\left(\Vstar_1 - \RealV{\phase}_1
  \right)(\MyState)\Big]^2.
\end{align*}
Thus, a standard $\sqrt{K}$ regret bound can be obtained as soon as
the term \MainTerm is bounded. The next step is then to transform
\MainTerm into an estimation problem through optimism.

Recall from equation~\eqref{eqn:Bonus}.  that we introduce optimism
via an exploration bonus of the form
\begin{align}
\label{eqn:bonusphase}
\bonusnew{\phase}_\hstep(\state, \action) & \defeq
\UncertaintyParam{\phase}_\hstep \| \phi_\hstep(\state, \action)
\|_{(\Covariance{\phase}_\hstep)^{-1}}
\end{align}
for all state-action pairs. This bonus is created in
\cref{line:bonusS4}, and passed to \qlearning{}.  The uncertainty
parameter $\UncertaintyParam{\phase}$ to be used in phase $\phase$ is
defined in \cref{eqn:Bonus}. To be clear, the covariance matrix used
to construct the bonus is the one in the current phase
$\Covariance{\phase}$. In order to proceed, we must relate the bonus
to the uncertainty function.  Let us define the reference covariance
\begin{align*}
\ExpectedCovariance{\phase}_\hstep \defeq \sum_{\jindex=1}^{\phase-1}
\nSamples{\jindex} \E_{\phi_\hstep\sim \Policy{\jindex}} [\phi_\hstep
  \phi_\hstep^\top] + \lambdareg\Identity.
\end{align*}
We also recall our earlier definition~\eqref{EqnDefnAlpha} of the
uncertainty parameter
\begin{align}
\label{eqn:ExpectedUncertaintyParam}
\ExpectedUncertaintyParam{\phase}_\hstep = \cbonus \Big \{ \sqrt{\dim
  \log \big(\tfrac{\dim \phase \nSamples{1:\phase}}{\delta}\big)} +
\sqrt{\lambdareg} \Big\}.
\end{align}

On our way to prove optimism, the next proposition highlights the
relation between the bonus and the uncertainty function.
\begin{lemma}[Bonus Bound]
\label{lem:BonusBound}
Set $\lambdareg = \Theta (\log\frac{\dim
  \phase\nSamples{1:\phase}}{\delta})$.  There exists a large enough
$\cbonus\in\R$ in \cref{eqn:ExpectedUncertaintyParam} such that with
probability at least $1-\delta$ jointly for all episodes $\nEpisodes$
we have the pointwise bound
\begin{align}
\UncertaintyFunction{\phase}_\hstep \leq \bonusnew{\phase}_\hstep \leq  
\cboundedopt \UncertaintyFunction{\phase}_\hstep
\end{align}
for some constant $\cboundedopt \in \R$.
\end{lemma}
The lemma is proved in \cref{sec:ProblemDependentBonus}, and it allows
us to claim that the algorithm is optimistic (using the left
inequality above) but without using a bonus that is too large (right
inequality above) which would create too much regret. We verify such
optimistic claim in \cref{sec:Optimism}. Recall the definition of
transfer error in \cref{eqn:TransferError} and the comparator error in
\cref{EqnDefnComparator} where $\Qtarget{\hstep+1} =
\Qfunction{\phase}_{\hstep+1}$ is the network returned by
\qlearning{}.
\begin{lemma}[Near-Optimism]
\label{lem:Optimism}
Suppose that the event in \cref{lem:BonusBound} holds jointly for all
episodes $\nEpisodes$. Then optimism holds in the following sense:
\begin{subequations}
\begin{align}
\Qfunction{\phase}_{\hstep}(\state,\action) \geq
\Qstar_\hstep(\state,\action) -\sum_{\tau=\hstep}^\horizon
\E_{(\MyState'_\tau,\MyAction'_\tau)\sim \pistar \mid
  (\state,\action)} \Comparator{(\phase)}{\tau} (\MyState'_\tau,
\MyAction'_\tau), \qquad \forall (\state,\action)
\end{align}
As a consequence, under the same event, we have
\begin{align}
\Estart \Vfunction{\phase}_1(\MyState) \geq \Estart\Vstar(\MyState) -
\terr.
\end{align}
\end{subequations}
\end{lemma}
From the optimism in the procedure, at any phase $\jindex$, we have
the bound
\begin{align*}
0 \leq \Estart\left(\Vstar_1 - \RealV{\jindex}_1\right)(\MyState) \leq
\Estart \left(\Vfunction{\jindex} _1 -
\RealV{\jindex}_1\right)(\MyState) + \terr,
\end{align*}
and hence
\begin{align*}
  \MainTerm & \lesssim \Term_{4a} + \Term_{4b} +
  \underbrace{\left(c\horizon\sump \sum_{\jindex=1}^{\phase-1}
    \nSamples{\jindex} + \sump \nSamples{\phase} \right)}_{=
    \nEpisodes} \terr^2
\end{align*}
where
\begin{align*}
\Term_{4a} \defeq c\horizon\sump \sum_{\jindex=1}^{\phase-1}
\nSamples{\jindex} \Big[\Estart \big(\Vfunction{\jindex}_1 -
  \RealV{\jindex}_1 \big)(\MyState)\Big]^2m \quad \mbox{and} \quad
\Term_{4b} \defeq \sump \nSamples{\phase}\Big[\Estart
  \big(\Vfunction{\phase}_1 - \RealV{\phase}_1 \big)(\MyState)\Big]^2.
\end{align*}
In turn bounding $\Term_{4a}$ and $\Term_{4b}$, we find that
\begin{align*}
\MainTerm & \lesssim \horizon \nPhases
\sum_{\jindex=1}^{\phase}\nSamples{\jindex}
\Big[\Estart\left(\Vfunction{\jindex}_1 - \RealV{\jindex}_1
  \right)(\state)\Big]^2 +\nEpisodes \terr^2.
\end{align*}

The remainder of the proof is devoted to deriving a high probability
bound on the first term of the above display.  Using the error bounds
on \qlearning{} from~\cref{thm:Q-learningExp}, we can write
\begin{align*}
\MainTerm & \lesssim \horizon \nPhases
\sum_{\jindex=1}^\phase\nSamples{\jindex}\Big[ \sumh
  \Esah{\Policy{\jindex}} (\ERMerror{\jindex}_\hstep +
  \bonusnew{\jindex}_\hstep - \Comparator{(\jindex)}{\hstep}
  )(\state_\hstep,\action_\hstep) \Big]^2 + \nEpisodes \terr^2 \\
& \lesssim  \horizon \nPhases  \sum_{\jindex=1}^\phase\nSamples{\jindex}\Big[ \sumh
\Esah{\Policy{\jindex}}
(\ERMerror{\jindex}_\hstep)(\state_\hstep,\action_\hstep) \Big]^2
+ \nEpisodes \terr^2 .
\end{align*}
The second step follows by bringing $\terr$ outside the square and by
bounding the bonus by using \fullref{lem:BonusBound}.

Putting together the pieces, we have
\begin{align*}
\MainTerm & \overset{(i)}{\lesssim} \horizon^2 \nPhases
\sum_{\jindex=1}^\phase\nSamples{\jindex}\sumh \Big[
  \Esah{\Policy{\jindex}} \ERMerror{\jindex}_\hstep
  (\state_\hstep,\action_\hstep) \Big]^2 + \nEpisodes \terr^2 \\
& \overset{(ii)}{\leq} \horizon^2 \nPhases
\sumh\sum_{\jindex=1}^\phase\nSamples{\jindex}
\Esah{\Policy{\jindex}}\Big[ \ERMerror{\jindex}_\hstep
  (\state_\hstep,\action_\hstep) \Big]^2 + \nEpisodes \terr^2 \\
& = \horizon^2 \nPhases \sumh\sum_{\jindex=1}^\phase\nSamples{\jindex}
\Esah{\Policy{\jindex}}\Big[ \ExpectedUncertaintyParam{\jindex}_\hstep
  \|\phi_\hstep(\state_\hstep,\action_\hstep)\|_{(\ExpectedCovariance{\jindex}_{\hstep})^{-1}}
  \Big]^2 +\nEpisodes \terr^2,
\end{align*}
where step (i) follows from the Cauchy--Schwarz inequality, and step
(ii) follows from Jensen's inequality.  Notice that that we have the
ordering $\ExpectedUncertaintyParam{1}_\hstep \leq \cdots \leq
\ExpectedUncertaintyParam{\phase}_\hstep \leq \dots \leq
\ExpectedUncertaintyParam{\nPhases}_\hstep$, i.e., the sequence of
reference parameter uncertainty must be non-decreasing.  Consequently,
we find that
\begin{align}
\label{eqn:MainTermIntermediate}
\MainTerm & \lesssim \horizon^2 \nPhases \sumh\left(
\ExpectedUncertaintyParam{\nPhases}_\hstep
\right)^2\sum_{\jindex=1}^\phase\nSamples{\jindex}
\Esah{\Policy{\jindex}}
\|\phi_\hstep(\state_\hstep,\action_\hstep)\|^2_{(\ExpectedCovariance{\jindex}_{\hstep})^{-1}}
+ \nEpisodes \terr^2.
\end{align}
We now proceed to bound the sum of the quadratic terms, a quantity
that arises in linear bandit analysis. In order to do so, we need to
define the triggering value used in \cref{line:NewPolicyTrigger} of
\cref{alg:Sequoia}.
\begin{align}
\label{eqn:TriggeringValue}
\TriggeringValue = \Theta\left(\log\frac{n\phase}{\delta}\right)
\end{align}
where $n \leq \nEpisodes$ is the number of times the condition has
been checked in phase $\phase$.  We obtain the following lemma which
is proved in \cref{sec:EllipticPotential}.
\begin{lemma}[Elliptic Potential]
\label{lem:EllipticPotential}
Assume $\lambdareg =
\Theta(\log\frac{\dim\phase\nSamples{1:\phase}}{\delta}) \geq 1$.
There exists a setting $\TriggeringValue$ as defined in
\cref{eqn:TriggeringValue} such that with probability at least
$1-\delta$ we have
\begin{align}
\sum_{\jindex=1}^{\phase}  \nSamples{\jindex} \Esah{\Policy{\jindex}} \|\phi_\hstep(\state_\hstep,\action_\hstep)\|_{(\ExpectedCovariance{\jindex}_\hstep)^{-1}}^2 \leq 8 \left(\TriggeringValue + 1 \right)  \InformationGain_\hstep.
\end{align}
\end{lemma}
Continuing the bound~\eqref{eqn:MainTermIntermediate}, recalling the
definition~\eqref{eqn:ExpectedUncertaintyParam} and applying
Cauchy-Schwartz gives us
\begin{align*}
\MainTerm & \lesssim \horizon^2 \nPhases \left( \sumh \left(
\UncertaintyParam{\nPhases}_\hstep \right)^2\InformationGain_\hstep
\right) \times 8 \left(\TriggeringValue + 1 \right) + \nEpisodes
\terr^2 \\ & \lesssim \horizon^2 \nPhases \sumh \left( \dim_\hstep +
\log\frac{\dim\nPhases \nEpisodes}{\delta}
\right)\InformationGain_\hstep \times 8 \left(\TriggeringValue + 1
\right) + \nEpisodes \terr^2.
\end{align*}

To conclude, it remains to bound the total number $\nPhases$ of
phases; we state the bound here and prove it in
\cref{sec:NumberOfPhases}.
\begin{lemma}[Number of Phases]
\label{lem:NumberOfPhases}
Under the conditions of \cref{lem:EllipticPotential}, the total number
of phases up to episode $\nEpisodes$ is upper bounded as
\begin{align}
\nPhases \leq \sumh \frac{\InformationGain_\hstep}{\log\left( 1 +
  \tfrac{1}{8} \TriggeringValue \right)}
\end{align}
with probability at least $1-\delta$.
\end{lemma}


\subsection{Proof of Lemma~\ref{lem:PhasedRegret}}
\label{sec:PhasedRegret}

The total regret up to episode $\nEpisodes$ can be expressed as a sum
of regret incurred in different phases
\begin{align}
\label{eqn:Regret2Regretphase}
\Regret(\nEpisodes) = \sump \Regretphase{\phase}.
\end{align}
Notice that in every phase $\phase$, the \qlearning{} procedure is
invoked with the
\prm{} $\sqlearningPolicyCover{\phase}$ which consists of the mixture
policy $\sum_{\jindex=1}^{\phase-1}
\nSamples{\jindex}\Policy{\jindex}$ and in addition \sqlearning{}
plays the policy $\Policy{\phase}$ between \cref{line:NewPolicyRepeat}
and \cref{line:NewPolicyTrigger} (in \cref{alg:Sequoia}) for exactly
$\nSamples{\phase}$ trajectories. Notice that in this proof the
sequence $\{\nSamples{\jindex}\}_1^\phase$ is assumed to be fixed,
i.e., non-random.

Outside of the call to \qlearning{}, \sqlearning{} induces a regret
exactly equal to
\begin{align}
\text{\sqlearning{}'s $\Regretphase{\phase} = \nSamples{\phase}\times
  \Estart \left(\Vstar_1 - \RealV{\phase}_1 \right)(\state)$}.
\end{align}  

In the same phase $\phase$, the regret due to the call to \qlearning{}
in \cref{line:SequoiaCallsQlearning} of \sqlearning{} is
\begin{align}
\text{\qlearning{}'s $\Regretphase{\phase} \propto \horizon
  \sum_{\jindex}^{\phase-1}\nSamplesRand{\jindex}\times \Estart
  \left(\Vstar_1 - \RealV{\jindex}_1 \right)(\state)$}
\end{align} 
where $\nSamplesRand{\jindex}$ is the random number of times that
policy $\Policy{\jindex}$ is actually played within \qlearning{} in
phase $\phase$.  Intuitively, $\nSamplesRand{\jindex} \approx
\nSamples{\phase}$ since $\E \nSamplesRand{\jindex} =
\nSamples{\phase}$. We make this precise by applying a Bernstein
inequality for martingales (cf. Thm. 1 from the
paper~\cite{beygelzimer2011contextual}).

Let $0 \leq X_\jindex \defeq \Es{\startdistribution} (\Vstar_1 -
\RealV{\jindex}_1)(\state') \leq 1$ be the random regret in step
$\jindex$ of \qlearning{}; here the randomness comes from the random
index $\jindex$ of the policy mixture. Let $n_{\text{tot}} =
\sum_{\jindex=1}^{\phase-1} \nSamples{\jindex} $; we can write
\begin{align*}
\sum_{t=1}^{n_\text{tot}} X_t =
\sum_{\jindex=1}^{\phase-1}\nSamplesRand{\jindex} \times
\Es{\startdistribution} \left(\Vstar_1 - \RealV{\jindex}_1
\right)(\state').
\end{align*}
By construction, the $X_t$'s are i.i.d., and $\E
\sum_{t=1}^{n_{\text{tot}}} X_t = \sum_{\jindex=1}^{\phase-1}
\nSamples{\jindex} \Es{\startdistribution} (\Vstar_1 -
\RealV{\jindex}_1)(\MyState')$.  Therefore, upon invoking Theorem 1
from the paper~\cite{beygelzimer2011contextual} and recalling that
$X^2_t \leq X_t$, we find that
\begin{align*}
  \sum_{t=1}^{n_{\text{tot}}} X_t & \leq \sum_{t=1}^{n_{\text{tot}}}
  \E X_t + 2\sqrt{\big(\sum_{t=1}^{n_{\text{tot}}} \E_t X_t \big)
    \log( \tfrac{1}{\delta})} + 2\log (\tfrac{1}{\delta})
\end{align*}
with probability at least $1-\delta$.  Completing the square on the
right hand side and applying Cauchy--Schwarz inequality yields the
upper bound $2 \sum_{t=1}^{n_{\text{tot}}} \E X_t +3 \log
\big(\tfrac{1}{\delta}\big)$.

Thus, the regret contributed by \qlearning{} in phase $p$ can be upper
bounded by the cumulative regret by \sqlearning{} (excluding its call
to \qlearning{})---viz.
\begin{align*}
\text{\qlearning{}'s $\Regretphase{\phase}$} & \lesssim \horizon
\sum_{\jindex}^{\phase-1}\nSamples{\jindex}\times \Estart
\left(\Vstar_1 - \RealV{\jindex}_1 \right)(\state) +
3\horizon\log\frac{1}{\delta}.
\end{align*} 
Summing together these contributions over all phases, and applying the
union bound over all possible phases yields the claim.


\subsection{Proof of \texorpdfstring{\fullref{lem:BonusBound}}{Problem Dependent Bonus}}
\label{sec:ProblemDependentBonus}

The main part of the proof is to show that the empirical covariance
matrices are accurate enough.  \sqlearning{} constructs the cumulative
covariance used to construct the bonus to be used in phase $\phase$ as
the sum of two terms: 1) the covariance estimate returned by
\qlearning{} in \cref{line:SequoiaCallsQlearning} of
\cref{alg:Sequoia} and 2) the increment obtained by \sqlearning{}
between \cref{line:NewPolicyRepeat} and \cref{line:NewPolicyTrigger}
of \cref{alg:Sequoia} . We can write:
\begin{align}
\label{eqn:InLemmaCovariance}
\Covariance{\phase}_\hstep = \underbrace{\lambdareg \Identity + \sum_{\iindex=1}^{\nSamples{1:\phase-1}} \phi_{\iindex\hstep}\phi_{\iindex\hstep}^\top}_{\text{\qlearning{}'s Covariance Estimate}} +  \underbrace{\sum_{\jindex=1}^{\nSamples{\phase}} \phi_{\jindex\hstep}\phi_{\jindex\hstep}^\top}_{\text{\sqlearning's increment}}.
\end{align}
For simplicity we have denoted with $\nSamples{1:\phase-1} =
\sum_{\iindex=1}^{\phase-1}\nSamples{\iindex}$; the first summation is
over the feature vectors $\{ \phi_\iindex \}$ sampled by \qlearning{}
and the second is over the feature vectors $\{ \phi_\jindex \}$
examined by \sqlearning. Notice that there exists a setting
$\lambdareg = \Theta(\log\frac{d}{\delta})$ that allows us to use
\fullref{prop:ConcentrationRegularizedCovariance} twice and claim with
probability at least $1-\delta/2$
\begin{align*}
\Covariance{\phase}_\hstep & = \frac{\lambdareg}{2} \Identity +
\sum_{\iindex=1}^{\nSamples{1:\phase-1}}
\phi_{\iindex\hstep}\phi_{\iindex\hstep}^\top + \frac{\lambdareg}{2}
\Identity + \sum_{\jindex=1}^{\nSamples{\phase}}
\phi_{\jindex\hstep}\phi_{\jindex\hstep}^\top \\ & \preceq
2\lambdareg\Identity + 2\nSamples{1:\phase-1}\E_{\phi \sim
  \sqlearningPolicyCover{\phase-1}} \phi_\hstep\phi_\hstep^\top +
2\nSamples{\phase}\E_{\phi \sim \Policy{\phase}}
\phi_\hstep\phi_\hstep^\top \\ & = 2\lambdareg\Identity +
2\nSamples{1:\phase}\E_{\phi \sim
  \sqlearningPolicyCover{\phase}}\phi_\hstep\phi_\hstep^\top \\ & =
2\ExpectedCovariance{\phase}_\hstep.
\end{align*}
(We have indicated with $\phi \sim \Policy{\phase}$ the random feature
sampled according to the policy mixture using the \prm{}).  We
conclude that under such event we must have
\begin{align*}
\bonusnew{\phase}_\hstep(\state, \action) =
\UncertaintyParam{\phase}_\hstep \| \phi_\hstep(\state, \action)
\|_{(\Covariance{\phase}_\hstep)^{-1}} \geq
\ExpectedUncertaintyParam{\phase}_\hstep\| \phi_\hstep(\state,
\action) \|_{(\ExpectedCovariance{\phase}_\hstep)^{-1}} &
=\UncertaintyFunction{\phase}_\hstep(\state, \action).
\end{align*}
and under the same event
\begin{align}
\bonusnew{\phase}_\hstep(\cdot,\cdot) =
\UncertaintyParam{\phase}_\hstep \| \phi_\hstep(\cdot,\cdot)
\|_{(\Covariance{\phase}_\hstep)^{-1}} \leq 2
\ExpectedUncertaintyParam{\phase}_\hstep\| \phi_\hstep(\cdot,\cdot)
\|_{(\ExpectedCovariance{\phase}_\hstep)^{-1}} & = \cboundedopt
\UncertaintyFunction{\phase}_\hstep(\cdot,\cdot).
\end{align}
A union bound over all possible phases and rescaling $\delta$
concludes.


\subsection{Proof of Lemma~\ref{lem:Optimism}}
\label{sec:Optimism}

When \qlearning{} terminates, \cref{thm:Q-learningExp} ensures it
returns a state-action value function $\Qbest{}$ such that
\begin{align*}
\Qbest{\hstep}(\state,\action) = \left( \T_\level \Qbest{\hstep+1}
\right)(\state,\action) + \err_{\hstep}(\state,\action), \qquad
\mbox{and} \\
\min\{0,\left(-\UncertaintyFunction{\phase}_\hstep+\bonusnew{\phase}_\hstep\right)(\state,\action)
\} - \Comparator{(\phase)}{\hstep}(\state,\action) \leq
\err_{\hstep}(\state,\action), 
\end{align*}
where both relations hold uniformly over all state-action pairs
$(\state, \action)$.  Conditioned on the event from
\cref{lem:BonusBound}, we have
\begin{align*}
-\Comparator{(\phase)}{\hstep}(\state,\action) = \min
\{0,-\UncertaintyFunction{\phase}_\hstep(\state,\action) +
\bonusnew{\phase}_\hstep(\state,\action) \}
-\Comparator{(\phase)}{\hstep}(\state,\action) \leq
\err_\hstep(\state,\action).
\end{align*}
which implies that $\err_{\hstep}(\state,\action) \geq
-\Comparator{(\phase)}{\hstep} (\state,\action)$ for all state-action
pairs and time steps $\hstep$.

Using this bound, we now perform backwards induction over the timestep
$\hstep$ in order to prove that
\begin{align}
  \label{EqnIndBound}
\Qhatstar_{\hstep+1}(\state_{\hstep+1},\action_{\hstep+1}) \geq
\Qstar_{\hstep+1}(\state_{\hstep+1},\action_{\hstep+1}) -
\sum_{\tau=\hstep+1}^\horizon
\E_{(\stateprime_\tau,\actionprime_\tau)\sim \pistar \mid
  (\state_{\hstep+1},\action_{\hstep+1})}
\Comparator{(\phase)}{\hstep}(\stateprime_\tau,\actionprime_\tau),
\qquad \forall (\state_{\hstep+1},\action_{\hstep+1})
\end{align}
For the base case $\hstep = \horizon$, all action-value functions are
zero, so that the bound~\eqref{EqnIndBound} certainly holds.

Now assume that the bound~\eqref{EqnIndBound} holds at timestep
$\hstep + 1$, for some $\hstep \in \{1, \ldots, \horizon-1 \}$; we
need to show that it also holds at timestep $\hstep$.  Fix an
arbitrary state-action pair $(\state, \action)$.  From our earlier
lower bound, we have
\begin{align*}
\Qbest{\hstep}(\state,\action) & = \left( \T_\hstep \Qbest{\hstep+1}
\right)(\state,\action) + \err_{\hstep}(\state,\action) \\ & \geq
\left( \T^{\pistar}_\hstep \Qbest{\hstep+1} \right)(\state,\action) +
\err_{\hstep}(\state,\action) \\
& \overset{\text{(i)}}{\geq} \left( \T^{\pistar}_\hstep
\Qstar_{\hstep+1} \right)(\state,\action) +
\err_{\hstep}(\state,\action) - \sum_{\tau=\hstep+1}^\horizon
\E_{(\stateprime_\tau,\actionprime_\tau)\sim \pistar \mid
  (\state,\action)}
\Comparator{(\phase)}{\hstep}(\stateprime_\tau,\actionprime_\tau) \\
& \geq \left( \T_\hstep \Qstar_{\hstep+1} \right)(\state,\action) -
\sum_{\tau=\hstep}^\horizon
\E_{(\stateprime_\tau,\actionprime_\tau)\sim \pistar \mid
  (\state,\action)}
\Comparator{(\phase)}{\hstep}(\stateprime_\tau,\actionprime_\tau).
\end{align*}
Here step (i) follows from the induction hypothesis.  Thus, we have
shown that the bound~\eqref{EqnIndBound} holds at timestep $\hstep$,
which completes our proof via induction.


\subsection{Proof of Lemma~\ref{lem:EllipticPotential}}
\label{sec:EllipticPotential}

We now prove the elliptic potential bound stated in
Lemma~\ref{lem:EllipticPotential}.  Let
$\phi^{(\jindex)}_{\iindex\hstep}$ be the $\iindex$ experienced
feature vector at level $\hstep$ that \sqlearning{} uses to check the
triggering condition in \cref{line:NewPolicyTrigger} during phase
$\jindex$.  Since $\lambdareg \geq 1$ we have $\| \phi
\|_{(\Covariance{\jindex}_\hstep)^{-1}} \leq 1, \; \forall \|\phi\|_2
\leq 1$. Thus, when the triggering condition holds, the condition
itself is not violated by much:
\begin{align}
\TriggeringValue\leq \sum_{\iindex=1}^{\nSamples{\jindex}} \|
\phi^{(\jindex)}_{\iindex \hstep}
\|^2_{(\Covariance{\jindex}_\hstep)^{-1}} \leq
\underbrace{\sum_{\iindex=1}^{\nSamples{\jindex}-1} \|
  \phi^{(\jindex)}_{\iindex \hstep}
  \|^2_{(\Covariance{\jindex}_\hstep)^{-1}}}_{\leq\TriggeringValue} +
\underbrace{\|\phi^{(\jindex)}_{\nSamples{\jindex} \hstep}
  \|^2_{(\Covariance{\jindex}_\hstep)^{-1}}}_{\leq 1} \leq
\TriggeringValue + 1.
\end{align}
Using \fullref{lem:ProportionalEstimate} and
\fullref{prop:ConcentrationRegularizedCovariance}, we find that with
probability at least $1-\delta$
\begin{align*}
\nSamples{\jindex} \Esah{\Policy{\jindex} }
\|\phi_\hstep(\state_\hstep,\action_\hstep)
\|_{(\ExpectedCovariance{\jindex}_\hstep)^{-1}}^2 & \leq 4
\nSamples{\jindex} \Esah{\Policy{\jindex} }
\|\phi_\hstep(\state_\hstep,\action_\hstep)
\|_{(\Covariance{\jindex}_\hstep)^{-1}}^2 \\ & \leq 8
\sum_{\iindex=1}^{\nSamples{\jindex}} \| \phi^{(\jindex)}_{\iindex
  \hstep} \|^2_{(\Covariance{\jindex}_\hstep)^{-1}} \leq 8
\left(\TriggeringValue + 1 \right) .
\end{align*}
Now recall that
\begin{align*}
\ExpectedCovariance{\jindex+1}_\hstep =
\ExpectedCovariance{\jindex}_\hstep + \nSamples{\jindex}
\Esah{\Policy{\jindex} }
\phi_\hstep(\state_\hstep,\action_\hstep)\phi_\hstep(\state_\hstep,\action_\hstep)^\top.
\end{align*}

Since $\left(\TriggeringValue + 1 \right) \geq e-1$, we can invoke
\fullref{LemInfoGain} so as to ensure that
\begin{align*}
\nSamples{\jindex} \Esah{\Policy{\phase} }
\|\phi_\hstep(\state_\hstep,\action_\hstep)
\|_{(\ExpectedCovariance{\jindex})^{-1}}^2 \leq 8
\left(\TriggeringValue + 1 \right) \log\frac{\det\left(
  \ExpectedCovariance{\jindex+1}_\hstep \right)}{\det\left(
  \ExpectedCovariance{\jindex}_\hstep \right)}.
\end{align*}
Summing over the phases and cancelling terms in the telescopic sum
yields
\begin{align*}
\sum_{\jindex=1}^{\phase} \nSamples{\jindex} \Esah{\Policy{\jindex}}
\|\phi_\hstep(\state_\hstep,\action_\hstep)\|_{(\ExpectedCovariance{\jindex}_\hstep)^{-1}}^2
& \leq 8 \left(\TriggeringValue + 1 \right) \log\frac{\det\left(
  \ExpectedCovariance{\phase+1}_\hstep \right)}{\det\left(
  \ExpectedCovariance{1}_\hstep \right)} \\ & \leq 8
\left(\TriggeringValue + 1 \right) \InformationGain_\hstep.
\end{align*}
A union bound over all possible phases concludes.


\subsection{Proof of \texorpdfstring{\fullref{lem:NumberOfPhases}}{Number of Phases}}
\label{sec:NumberOfPhases}

For invertible matrices $A$ and $B$, note that $A \preceq B$ implies
that $A^{-1} \succeq B^{-1}$, and moreover, we have the equivalence $A
\preceq B \Longleftrightarrow x^\top A x \leq x^\top B x$ for all $x
\in \real^d$.  We now combine Lemma~\ref{lem:ProportionalEstimate}
with Proposition~\ref{prop:ConcentrationRegularizedCovariance} so as
to argue that \cref{alg:Sequoia} makes sufficient progress.  In
particular, consider each time that \cref{line:NewPolicyTrigger} of
\cref{alg:Sequoia} triggers a new phase at level $\hstep$.  Then with
probability at least $1 - \delta$, we must have
\begin{align*}
\nSamples{\phase} \Esah{\Policy{\phase}} \|
\phi_\hstep(\state_\hstep,\action_\hstep)
\|^2_{(\ExpectedCovariance{\phase}_\hstep)^{-1}} & \geq \frac{1}{4}
\nSamples{\phase} \Esah{\Policy{\phase}} \|
\phi_\hstep(\state_\hstep,\action_\hstep)
\|^2_{(\Covariance{\phase}_\hstep)^{-1}} \\ & \geq \frac{1}{8}
\sum_{\iindex=1}^{\nSamples{\phase}} \| \phi^{(\phase)}_{\iindex
  \hstep} \|^2_{(\Covariance{\phase}_\hstep)^{-1}} \\ & \geq
\frac{1}{8} \TriggeringValue.
\end{align*}
When this bound holds and the the triggering condition is satisfied at
level $\hstep$ in phase $\phase$, then the information ratio must
increase by a constant fraction: more precisely, \cref{LemInfoGain}
guarantees that
\begin{align*}
\frac{\det(\ExpectedCovariance{\phase+1}_\hstep)}{\det
  \ExpectedCovariance{\phase}_\hstep} & =
\frac{\det\left(\ExpectedCovariance{\phase}_\hstep + \nSamples{\phase}
  \Esah{\Policy{\phase}
  }\phi_\hstep(\state_\hstep,\action_\hstep)\phi_\hstep(\state_{\hstep},\action_{\hstep})^\top\right)}{\det
  \ExpectedCovariance{\phase}_\hstep} \\ & \geq 1+ \nSamples{\phase}
\Esah{\Policy{\phase}} \| \phi_\hstep(\state_\hstep,\action_\hstep)
\|^2_{(\ExpectedCovariance{\phase}_\hstep)^{-1}} \\ & \geq
1+\frac{1}{8}\ControllerTriggeringValue{\deltaPhase}.
\end{align*}
By induction, in phase $\phase$ we must have (notice that we are
ignoring the additional contribution that arises when level $\hstep$
is not the one that triggers a new phase)
\begin{align}
\frac{\det \ExpectedCovariance{\phase+1}_\hstep}{\det \ExpectedCovariance{1}_\hstep} \geq \left( 1+\frac{1}{8}\TriggeringValue\right)^{\ControllerNSwithces{\phase}{\hstep}}
\end{align}
where $\ControllerNSwithces{\phase}{\hstep}$ is the number of switches up to phase $\phase$ that were triggered at level $\hstep$. Taking $\log$ gives
\begin{align}
\ControllerNSwithces{\phase}{\hstep} \leq \frac{\log\frac{\det \ExpectedCovariance{\phase+1}_\hstep}{\det \ExpectedCovariance{1}_\hstep}}{\log\left( 1+\frac{1}{3}\TriggeringValue\right) }. 
\end{align}
Recalling the total number of switches across levels equals the total
number of phases, i.e., $\sumh \ControllerNSwithces{\phase}{\hstep} =
\nPhases$, together with the relevant union bound over phases
concludes.

\section{Auxiliary results}
\label{sec:Qexp}

In this appendix, we collect together various auxiliary results
that we use in our main argument, along with their proofs.


\subsection{Information Gain}

\begin{lemma}[Upper Bound on Information Gain]
\begin{align}
\log\left( \frac{\det \ExpectedCovariance{\nPhases+1}_\hstep
}{\det\ExpectedCovariance{1}_\hstep}\right) \leq
\InformationGain_{\hstep} \leq \dim \log
\frac{\nEpisodes}{\dim\lambdareg}.
\end{align}
\end{lemma}

We must now bound the determinant of
$\ExpectedCovariance{\nPhases}_\hstep$ to compute the maximum number
of phases; we proceed in a way similar to Lemma 10 of \cite{Abbasi11},
the only difference being that the increments are not rank one.  Let
$\alpha_1,\dots,\alpha_\dim$ be the eigenvalues of
$\ExpectedCovariance{\nPhases}_\hstep$.  We must have
\begin{align}
\det \left( \ExpectedCovariance{\nPhases}_\hstep \right) =
\prod_\iindex \alpha_\iindex \leq \left( \frac{\sum_\iindex
  \alpha_\iindex}{\dim} \right)^\dim = \left(\frac{\Tr
  \ExpectedCovariance{\nPhases}_\hstep}{d} \right)^\dim.
\end{align}
We can upper bound the trace as follows:
\begin{align*}
\Tr \left(\ExpectedCovariance{\nPhases}_\hstep\right) =
\sum_{\phase=1}^{\nPhases}\Tr\left(\nSamples{\phase}
\CovarianceExpectedIncrement{\phase}_\hstep \right) & =
\sum_{\phase=1}^{\nPhases} \nSamples{\phase} \Tr
\left(\Esah{\Policy{\phase}}
\phi_\hstep(\state_\hstep,\action_\hstep)\phi_\hstep(\state_\hstep,\action_\hstep)^\top
\right) \\
& = \sum_{\phase=1}^{\nPhases} \nSamples{\phase}
\Esah{\Policy{\phase}} \Tr
\left(\phi_\hstep(\state_\hstep,\action_\hstep)\phi_\hstep(\state_\hstep,\action_\hstep)^\top
\right) \\
& \leq \sum_{\phase=1}^{\nPhases} \nSamples{\phase} \\
&
\leq \nEpisodes.
\end{align*}
Combining with the prior displays and recalling that
$\ExpectedCovariance{0}_\hstep = \lambdareg \Identity$ yields the
claim.


\subsection{Bounds on the information gain}
\label{SecBoundsInfoGain}

\begin{lemma}[Information gain bounds]
\label{LemInfoGain}
For any random vector $\phi \in \real^d$, scalar $\alpha > 0$ and
positive definite matrix $\Sigma$, we have the upper bound
\begin{subequations} 
\begin{align}
\label{EqnInfoUpper}    
\log \frac{\det(\Sigma + \alpha \E [\phi \phi^\top])}{\det \Sigma} &
\leq \alpha \E \| \phi \|^2_{\Sigma^{-1}}.
\end{align}
Moreover, we have the lower bounds
\begin{align}
\label{EqnInfoLower}
\log \frac{\det(\Sigma + \alpha \E [\phi \phi^\top])}{\det(\Sigma)}
\geq \log \left( 1+\alpha \E \|\phi \|^2_{\Sigma^{-1}} \right)
\overset{\text{(i)}}{\geq} \frac{\alpha}{L} \E \|\phi
\|^2_{\Sigma^{-1}},
\end{align}
where (i) holds whenever $\alpha \E \| \phi \|^2_{\Sigma^{-1}} \leq
L$ for some $L \geq e-1$,
\end{subequations}
\end{lemma}

\begin{proof}
We first begin with equivalent expression for the determinant ratio.
Letting $\lambda_j(M)$ denote the $j^{th}$-ordered eigenvalue of a
matrix $M$, we have
\begin{align}
\frac{\det(\Sigma + \alpha \E \phi\phi^\top)}{\det \Sigma} =
\det(\Identity + \alpha \Sigma^{-\frac{1}{2}}\E [\phi\phi^\top]
\Sigma^{-\frac{1}{2}}) & = \prod_{j=1}^d \lambda_j \left(\Identity +
\alpha \Sigma^{-\frac{1}{2}}\E \phi\phi^\top \Sigma^{-\frac{1}{2}}
\right) \nonumber \\
\label{eqn:detratio}
& = \prod_{j=1}^d \left(1 + \alpha \lambda_j \left(
\Sigma^{-\frac{1}{2}}\E \phi\phi^\top \Sigma^{-\frac{1}{2}}\right)
\right).
\end{align}

\myparagraph{Proof of the upper bound~\eqref{EqnInfoUpper}} Taking
logarithms in equation~\eqref{eqn:detratio} and using the elementary
bound $\log(1 + t) \leq t$, valid for $t \geq 0$, we have
\begin{align*}
  \log \Big( \frac{\det(\Sigma + \alpha \E \phi\phi^\top)}{\det
    \Sigma} \Big) = \sum_{j=1}^d \log \left(1 + \alpha \lambda_j
  \left( \Sigma^{-\frac{1}{2}}\E \phi\phi^\top
  \Sigma^{-\frac{1}{2}}\right) \right) & \leq \sum_{j=1}^d \Big \{
  \alpha \lambda_j \big( \Sigma^{-\frac{1}{2}}\E \phi\phi^\top
  \Sigma^{-\frac{1}{2}} \big) \Big \}  \\
& \stackrel{(i)}{=} \alpha \Tr \Big( \Sigma^{-\frac{1}{2}}\E
  \phi\phi^\top \Sigma^{-\frac{1}{2}} \Big) \\
& \stackrel{(ii)}{=} \alpha \E \| \phi \|^2_{\Sigma^{-1}},
\end{align*}
where step (i) follows since the trace is equal to the sum
of the eigenvalues, and step (ii) follows from cyclic properties
of the trace operator, and some algebra.

\medskip

\myparagraph{Proof of the lower bound~\eqref{EqnInfoLower}} In order
to prove the lower bound, we again begin with
equation~\eqref{eqn:detratio}. Notice that the eigenvalues are all
non-negative since the matrix $\Sigma^{-\frac{1}{2}} \E
[\phi \phi^\top] \Sigma^{-\frac{1}{2}}$ is positive semidefinite. Thus,
we can ignore the higher-order terms in the product so as to obtain
the lower bound
\begin{subequations}
\begin{align}
\label{EqnInterLower}
\frac{\det(\Sigma + \alpha \E \phi\phi^\top)}{\det \Sigma} & \geq 1 +
\alpha \sum_{j=1}^d \lambda_j \left( \Sigma^{-\frac{1}{2}}\E
\phi\phi^\top \Sigma^{-\frac{1}{2}}\right) \; = \; 1 + \alpha \E \|
\phi \|^2_{\Sigma^{-1}},
\end{align} 
where the final equality follows by the same sequence of calculations
as in step (ii) above.

In order to complete the proof, observe that $f(x) = \log(1 + x)$ is a
concave function.  Thus, for any $a \geq e - 1$ and $x \in [0, a]$, we
can set $\lambda = \tfrac{x}{a} \in [0,1]$, and write
\begin{align}
\label{EqnLogLower}
\log(1 + x) \; = \; f(x) = f \big(\lambda a + (1-\lambda) 0 \big)
\stackrel{(iii)}{\geq} \lambda f(a) + (1-\lambda) f(0) \; = \; \lambda
f(a) \stackrel{(iv)}{\geq}
\lambda = \tfrac{x}{a},
\end{align}
\end{subequations}
where step (iii) follows from Jensen's inequality; and step (iv) is
valid for any $a \geq e - 1$.

Finally, taking logarithms in inequality~\eqref{EqnInterLower} and
applying the lower bound~\eqref{EqnLogLower} yields
\begin{align*}
\log \frac{\det(\Sigma + \alpha \E \phi\phi^\top)}{\det \Sigma} \geq
\log \left( 1+\alpha \E \|\phi \|^2_{\Sigma^{-1}} \right) \geq
\frac{\alpha}{L} \E \|\phi \|^2_{\Sigma^{-1}},
\end{align*}
as claimed.
\end{proof}

\subsection{Proportional estimates under the triggering condition}

\newcommand{\ShatEmp}{\ensuremath{\widehat{S}_\numobs}}
Suppose that the triggering condition holds, so that we have the lower
bound
\begin{align}
\label{EqnMyTrigger}
\sumi Z_\iindex \geq 32\csn{\frac{\delta}{2 \numobs^2}} +
8\cn{\frac{\delta}{2 \numobs^2}} \defeq
\ControllerTriggeringValue{\delta}.
\end{align}
The following lemma shows that under this condition, the sample
average $\ShatEmp \defeq \tfrac{1}{\numobs} \sumi Z_\iindex$ is close
to the expectation $\E[Z]$.
\begin{lemma}[Proportional Estimates]
\label{lem:ProportionalEstimate}
Under the triggering condition~\eqref{EqnMyTrigger}, for any $\delta
\in (0,1)$, we have the sandwich result
\begin{align}
\label{EqnSandwich}
\tfrac{1}{2} \ShatEmp \leq \E[Z] \leq \tfrac{3}{2} \ShatEmp \qquad
\mbox{with prob. at least $1 - \delta$.}
\end{align}
\end{lemma}

In order to prove this claim, we first show that for any fixed
$\numobs$ at which the triggering conditioning holds, the sandwich
bound~\eqref{EqnSandwich} holds with probability at least $1 -
\delta'$, where $\delta' = \tfrac{\delta}{2 \numobs^2}$.  We can then
take a union bound over all $\numobs = 1, 2, \ldots$ to conclude that
for any $\numobs$, sandwich bound~\eqref{EqnSandwich} holds with
probability at least \mbox{$1 - \sum_{\numobs=1}^\infty
  \frac{\delta}{2 \numobs^2} \geq 1- \delta$,} as required.

Thus, for the remainder of the proof, we study the problem for a fixed
sample size $\numobs$.  Our proof is based two auxiliary results.
First, for i.i.d.  random variables $\{Z_i\}_{i=1}^\numobs$ taking
values in $[0,1]$, the sample average $\ShatEmp$ satisfies the
following empirical Bernstein bound: for any $\delta \in (0,1)$,
\begin{align}
\label{EqnEmpBern}  
\Pro \left[ \big|\ShatEmp - \E[Z] \big| \leq \sqrt{\tfrac{\csn{\delta}
      \widehat \Var Z }{n}} + \tfrac{\cn{\delta}}{n-1} \right] & \geq
1-\delta
\end{align}	
where $\csn{\delta} = 2 \log (\tfrac{4}{\delta})$; $\cn{\delta} =
\frac{7}{3} \log (\tfrac{4}{\delta})$, and
\begin{align*}
\widehat \Var Z = \frac{1}{n(n-1)} \sum \limits_{1 \leq i < j \leq
  n}(Z_i -Z_j)^2
\end{align*}
is the empirical variance.  This claim is a consequence of two
applications of the empirical Bernstein bound from the
paper~\cite{MP09}, as applied to the random variables $Z$ and then
$1-Z$, followed by a union bound to obtain the two-sided claim give
here.

Next, we observe that the sample variance can be upper bounded as
\begin{align}
\label{EqnVarBound}
\widehat \Var Z \defeq \frac{1}{n-1} \left(\sumi Z_i^2 - \sumi
Z_i\right) \leq \frac{1}{n-1} \sumi Z_i^2 \leq \frac{1}{n-1} \sumi Z_i
\leq \tfrac{\numobs}{\numobs-1} \; \leq \; 2,
\end{align}	
using the fact that each $Z_i \in [0,1]$.

Now combining the empirical Bernstein bound~\eqref{EqnEmpBern} with
the variance bound~\eqref{EqnVarBound},
we find that
\begin{align}
\label{eqn:Intermediate101}
\big|\ShatEmp - \Exs[Z] \big| & \leq \tfrac{1}{\numobs} \Big \{
\sqrt{2 \ShatEmp \csn{\delta'}} + 2\cn{\delta'} \Big \} \qquad
\mbox{with prob. at least $1 - \delta'$.}
\end{align}
The triggering condition~\eqref{EqnMyTrigger} ensures that $\numobs
\ShatEmp \geq 32\csn{\delta'} + 8\cn{\delta'}$, whence
$\tfrac{1}{\numobs} \csn{\delta} \leq \frac{\ShatEmp}{32}$ and
$\tfrac{1}{\numobs} \cn{\delta} \leq \frac{\ShatEmp}{8}$.
Putting together the pieces, we find that
\begin{align*}
|\ShatEmp - \Exs[Z]| \leq \tfrac{1}{\numobs} \big \{ \sqrt{2 \ShatEmp
  \csn{\delta}} + 2\cn{\delta} \big \} & \leq
\sqrt{\frac{(\ShatEmp)^2}{16}} + \frac{\ShatEmp}{4} \; = \;
\frac{\ShatEmp}{2} 
\end{align*}
\mbox{with probability at least $1- \delta'$.}  Re-arranging shows
that we have the sandwich relation $\tfrac{1}{2} \ShatEmp \leq \Exs[Z]
\leq \tfrac{3}{2} \ShatEmp$ with probability $1 - \delta'$, as
claimed.


\subsection{Non-Isotropic Proportional Estimates of the Empirical Covariance}

An important step in the analysis is ensuring that the empirical
covariance matrices computed by the algorithm are sufficiently close
to their (conditional) expectations.  In this section,
we discuss how to use matrix Chernoff techniques to establish
the requisite bounds.

Let $\{Z_\kindex \}_{\kindex=1}^\numobs$ be a sequence of independent,
symmetric and positive definite random matrices of dimension $\dim$.
$\dim$. Suppose that
\begin{align}
  \label{EqnZass}
0 \leq \lambda_{\min}(Z_\kindex), \qquad \lambda_{\max}(Z_\kindex)
\leq L_Z, \qquad \mbox{for all $\kindex = 1, \ldots, \nEpisodes$.}
\end{align}
The following result provides bounds on the sum $W = \sumk Z_\kindex$.
\begin{proposition}[Concentration of Regularized Covariance]
\label{prop:ConcentrationRegularizedCovariance}
Under the above conditions, for any $\delta \in (0,1)$ and $\lambda
\geq 2 L_Z \frac{\log\frac{2\dim}{\delta}}{\log\frac{36}{35}}$, we
have
\begin{align}
\label{eqn:RegularizedCovarianceSecondStatement}
\frac{1}{2}\left(W + \lambda\Identity \right)\preceq \E W +
\lambda\Identity \preceq \frac{3}{2}\left( W + \lambda\Identity
\right).
\end{align}
with probability at least $1 - \delta$.
\end{proposition}

\begin{proof}
In fact, we
establish a somewhat more general claim: namely, for any $\epsilon >
0$ and $\lambda > 0$, we have
  \begin{align}
\label{eqn:CovarianceConcentration}
\Pro\Big( \frac{1}{1 + \epsilon}\left(W + \lambda\Identity
\right)\preceq \E W + \lambda\Identity \preceq
\frac{1}{1-\epsilon}\left( W + \lambda\Identity \right) \Big) \geq
1-2\dim\left( 1-\frac{\epsilon^2}{4}
\right)^{\frac{\lambda}{L_Z+\frac{\lambda}{K}}}.
  \end{align}
  In order to recover the stated
  claim~\eqref{eqn:RegularizedCovarianceSecondStatement} we fix
  $\epsilon = \frac{1}{3}$. On one hand, if $L_Z \leq
  \frac{1}{2}\frac{\lambda}{K}$, then we have the (deterministic)
  sandwich relations
\begin{align*}
0 \preceq W \preceq \tfrac{1}{2} \lambda \Identity, \quad \mbox{and}
\quad 0 \preceq \E W \preceq \tfrac{1}{2} \lambda \Identity
\end{align*}
so that the bound~\eqref{eqn:RegularizedCovarianceSecondStatement}
holds deterministically.  On the other hand, if $L_Z \geq
\frac{1}{2}\frac{\lambda}{K}$, then the claim follows by choosing
$\lambda \geq 2 L_Z \frac{\log(\frac{2\dim}{\delta})}{\log
  (\frac{36}{35})}$.

In order to prove the bound~\eqref{eqn:CovarianceConcentration}, we
make use of the following matrix Chernoff inequality:
\begin{lemma}[Matrix Chernoff]
\label{lem:Tropp}
Consider the sum $Y = \sum_{\kindex=1}^K X_\kindex$ of a sequence
$\{X_\kindex \}_{\kindex \geq 1}$ of independent, symmetric PSD
matrices whose eigenvalues all lie in the interval $[0, L]$, and
suppose that $\Exs[Y] = I$.  Then we have
\begin{align}
\label{EqnTroppOne}    
(1-\epsilon) \Identity \preceq Y \preceq (1+\epsilon) \Identity
\end{align}
with probability at least $1 - 2 d \left(1-\frac{\epsilon^2}{4}
\right)^{\frac{1}{L}}$ for all $\epsilon \in (0,1)$.
\end{lemma}
This claim follows by applying Theorem 5.1.1 from
Tropp~\cite{tropp2015} twice, for the upper and lower tail
respectively, combined with the inequalities
\begin{align*}
\frac{e^{-\epsilon}}{(1-\epsilon)^{1-\epsilon}} \leq
1-\frac{\epsilon^2}{4}, \quad \mbox{and} \quad
\frac{e^{\epsilon}}{(1+\epsilon)^{1+\epsilon}} \leq
1-\frac{\epsilon^2}{4}, \quad \mbox{valid for any $\epsilon \in
  [0,1)$,}
\end{align*}
along with the fact that $a \leq b$ implies that $a^x \leq b^x$ for
all strictly positive scalars $a, b, x$.

Using~\cref{lem:Tropp}, we can now prove the
bound~\eqref{eqn:CovarianceConcentration}.  We define ``regularized''
versions of $Z_\kindex$ and $W$ via $X'_\kindex \defeq Z_\kindex +
\frac{\lambda}{\Kindex}\Identity$ and $Y' \defeq \sumk X'_\kindex$.
By definition, we have
\begin{align}
\label{eqn:YprimeDef}
Y' = \sumk \left( Z_\kindex + \frac{\lambda}{\Kindex} \Identity
\right) = W + \lambda \Identity \quad \mbox{and} \quad \E Y' = \E W +
\lambda \Identity.
\end{align}
Thus, in order to prove the claim, it suffices to establish a high
probability bound on the event
\begin{align}
\label{eqn:ConcentrationIntermediate}
\Event & \defeq \Big \{ (1-\epsilon) Y' \preceq \E Y' \preceq
(1+\epsilon) Y' \Big \}.
\end{align}
Since $\lambda_{\min}(\E Y') \geq \lambda$, the matrix $\E Y'$ is
strictly positive definite, and the matrix $\left( \E
Y'\right)^{-\frac{1}{2}}$ exists. We use it to define the new matrices
\begin{align}
\label{eqn:IntermediateXYdef}
X_\kindex & \defeq \left( \E Y'\right)^{-\frac{1}{2}} X'_\kindex
\left( \E Y'\right)^{-\frac{1}{2}}, \quad \mbox{and} \\
Y & \defeq \sumk X_\kindex = \left( \E Y'\right)^{-\frac{1}{2}} \left(
\sumk X'_\kindex \right) \left( \E Y'\right)^{-\frac{1}{2}} = \left(
\E Y'\right)^{-\frac{1}{2}} \left(Y' \right) \left( \E
Y'\right)^{-\frac{1}{2}}. \notag
\end{align}
Note that we have $\E[Y] = I$ by construction, so that the matrix
Chernoff bound~\eqref{EqnTroppOne} can be applied.  We observe that
\begin{align*}
\lambda_{\max}(X_\kindex) = \| X_\kindex \|_2 \leq \| \left( \E
Y'\right)^{-\frac{1}{2}} X'_\kindex \left( \E Y'\right)^{-\frac{1}{2}}
\|_2 \leq \| \left( \E Y'\right)^{-\frac{1}{2}} \|_2 \| X'_\kindex
\|_2 \|\left( \E Y'\right)^{-\frac{1}{2}} \|_2 \leq
\frac{1}{\lambda}\left( L_Z+\frac{\lambda}{K} \right) \defeq L.
\end{align*}
Applying the bound~\eqref{EqnTroppOne} yields
\begin{align*}
(1-\epsilon) \Identity \preceq Y \preceq (1+\epsilon) \Identity.
\end{align*}
with the stated probability.  Finally, we can pre- and post-multiply
by $(\E Y')^{\frac{1}{2}}$ and then use \cref{eqn:IntermediateXYdef}
so as to obtain
\begin{align*}
(1-\epsilon) \E Y' \preceq (\E Y')^{\frac{1}{2}} Y (\E
  Y')^{\frac{1}{2}} = Y' \preceq (1+\epsilon) \E Y'.
\end{align*}
Recalling the definition~\eqref{eqn:YprimeDef} of $Y'$, we see that
this sandwich is equivalent to the stated
claim~\eqref{eqn:RegularizedCovarianceSecondStatement}.
\end{proof}

\subsection{Concentration of Log Determinants}

In this section, we prove the following claim:
\begin{lemma}[Concentration of Log Determinants]
\label{lem:LogDet}
Let $\{x_\iindex\}$ be i.i.d. vector random variables from some
distribution such that $\| x_\iindex \|_2 \leq 1$.  If $\lambda
\gtrsim \log(\frac{dn}{\delta}) \geq 1$ and $G_1 \succeq \lambda
\Identity$ then with probability at least $1-\delta$ jointly for all
$n=1,2,\dots$ it holds that
\begin{align*}
\frac{1}{4} \log \frac{\det(G_1+ n\E xx^\top)}{\det G_1} - (8\sqrt{2}
+ 4)\log\frac{8n^2}{\delta} \leq \log \frac{\det\left( G_1 + \sumi
x_ix_i^\top \right)}{\det G_1 } \leq 8 \log \frac{\det(G_{1} + n\E
xx^\top) }{\det G_1 } + 8\log\frac{8n^2}{\delta}.
\end{align*}
\end{lemma}

Let us now prove it.  Let $G_\iindex = G_1 + \sum_{\jindex =
1}^{\iindex-1} x_\jindex x_\jindex^\top$.  Using Lemma 11 in
\cite{Abbasi11} we can write
\begin{align*}
\frac{1}{2}\sumi \| x_\iindex \|^2_{G^{-1}_\iindex}  \leq \log \frac{\det\left( G_0  + \sumi x_ix_i^\top \right)}{\det G_0 } \leq \sumi \| x_\iindex \|^2_{G^{-1}_\iindex}.
\end{align*}
Thus, we will now focus on bounding the sums of the quadratic
functions.  \cref{prop:ConcentrationRegularizedCovariance} together
with a double union bound over $n$ ensures that if $\lambda \gtrsim
\log(\frac{dn}{\delta})$ then for all $n$ with probability at least
$1-\delta/2$, we have
\begin{align*}
\sumi \| x_\iindex \|^2_{G^{-1}_\iindex} = \sumi x_\iindex
G^{-1}_\iindex x_\iindex & \leq 2 x_\iindex \Big( G_1 +
\E\sum_{\jindex=1}^{\iindex-1} x_\jindex x_\jindex^\top
\Big)^{-1}x_\iindex \\
& = 2 x_\iindex \Big( \underbrace{G_1 + (\iindex -1)\E x x^\top}_{= \E
  G_{i}} \Big)^{-1}x_\iindex \; = \; 2 \| x_i\|^2_{(\E G_{i})^{-1}}.
\end{align*}
Similarly, under the same event specified by
\cref{prop:ConcentrationRegularizedCovariance}, we have
\begin{align*}
\sumi \| x_\iindex \|^2_{G^{-1}_\iindex} = \sumi x_\iindex
G^{-1}_\iindex x_\iindex & \geq \frac{1}{2} x_\iindex \Big( G_1 +
\E\sum_{\jindex=1}^{\iindex-1} x_\jindex x_\jindex^\top
\Big)^{-1}x_\iindex \\ & = \frac{1}{2} x_\iindex \Big( \underbrace{G_1
  + (\iindex -1)\E x x^\top}_{= \E G_{i}} \Big)^{-1}x_\iindex \; = \;
\frac{1}{2} \| x_i\|^2_{(\E G_{i})^{-1}}.
\end{align*}

Now consider the random variable $X_i = \| x_i\|^2_{(\E G_{i})^{-1}} -
\E \| x_i\|^2_{(\E G_{i})^{-1}}$.  Since $\E G_\iindex \succeq
\Identity$, we have
\begin{align*}
\E X^2_i \leq \E \| x_i\|^4_{(\E G_{i})^{-1}} \leq \E \| x_i\|^2_{(\E G_{i})^{-1}} \leq 1.
\end{align*}
Applying a Bernstein martingale inequality (cf. Theorem 1 from the
paper~\cite{beygelzimer2011contextual}) and combining with the union
bound yields
\begin{align*}
\Bigg| \sumi X_i \Bigg| \leq 2 \sqrt{\left(\sumi \E \| x_i\|^2_{(\E
    G_{i})^{-1}} \right) \log (\tfrac{8 n^2}{\delta})} + 2
\log(\tfrac{8 n^2}{\delta})
\end{align*}
with probability at least $1-\delta/2$.

It remains to bound the sum of the predictable expectations under
square root. Coupling \fullref{LemInfoGain} with
\fullref{LemInfoGain} under the conditions $\alpha = 1, L =
2 > e-1$ we obtain
\begin{align}
\label{eqn:QsumBound}
\log \frac{\det(\E  G_\iindex +  \E x_\iindex x_\iindex ^\top)}{\det \E  G_\iindex} \leq \E \|x_i \|^2_{(\E G_\iindex)^{-1}} \leq 2\log \frac{\det(\E  G_\iindex +  \E x_\iindex x_\iindex ^\top)}{\det \E  G_\iindex }.
\end{align}
Summing over $\iindex \in [n]$, recalling $\E  G_{\iindex+1} = \E  G_\iindex +  \E x_\iindex x_\iindex^\top$ and cancelling the terms in the telescoping sum gives
\begin{align*}
\log \frac{\det \E G_{n+1}}{\det   G_1 } \leq \sumi \E \|x_i \|^2_{(\E G_\iindex)^{-1}} \leq 2\log \frac{\det \E G_{n+1}}{\det   G_1 }.
\end{align*}
We are now ready to show the upper bound. Removing the absolute value, using \cref{eqn:QsumBound} to bound the quadratic sum one obtains with probability $1-\delta$
\begin{align*}
\sumi \| x_\iindex \|^2_{G^{-1}_\iindex}
& \leq 4\log \frac{\det(\E G_{n+1}) }{\det G_1 } + 4 \sqrt{ 2\log \frac{\det(\E G_{n+1}) }{\det G_1 } \log\frac{8n^2}{\delta}} + 4\log\frac{8n^2}{\delta} \\
& \leq 8\log \frac{\det(\E G_{n+1}) }{\det G_1 } +8\log\frac{8n^2}{\delta}.
\end{align*}
The last inequality follows from completing the square and using Cauchy-Schwartz to simplify the statement.

We show the lower bound on the similar fashion. By lifting the
absolute value one obtains with probability at least $1-\delta$
\begin{align*}
\sumi \| x_\iindex \|^2_{G^{-1}_\iindex} & \geq \frac{1}{2}\log
\frac{\det(\E G_{n+1}) }{\det G_1 } - 4 \sqrt{ 2\log \frac{\det(\E
    G_{n+1}) }{\det G_1 } \log\frac{8n^2}{\delta}} -
4\log\frac{8n^2}{\delta} \\
& \geq \frac{1}{4} \log \frac{\det(\E G_{n+1})}{\det G_1} - (8\sqrt{2}
+ 4)\log\frac{8n^2}{\delta}.
\end{align*}


\subsection{Constrained Loss Lemmas}
In this section, let $\E$ denote the expectation operator for a pair
$(X, Y) \sim \mu$, where $x \in\R^d$ and $y \in \R$. Assume that the
second moment matrix $\E[X X^\top]$ is strictly positive
definite. Define the loss function $\L(\theta) =
\E(\smlinprod{X}{\theta} -Y)^2$, along with the unconstrained
minimizer $\theta^\star \defeq \arg \min_{\theta \in \real^d}
\L(\theta)$.  Our first result gives an equivalent expression for the
excess loss $\L(\theta) - \L(\theta^\star)$.
\begin{lemma}[Excess Loss]
\label{lem:ExcessLoss}
The excess loss can be written as
\begin{align}
\L(\theta) - \L(\theta^\star) = \| \theta - \theta^\star \|^2_{\E[X
    X^\top]}.
\end{align}
\end{lemma}
\begin{proof}
Since the loss is a strongly convex quadratic function, the minimizer
must satisfy the zero-gradient condition
\begin{align}
\label{eqn:OptimalityConditionUnconstrained}
0 & = \tfrac{1}{2} \nabla_\theta\L(\theta^\star) = \evalat{ \E [X
    (\smlinprod{X}{\theta} - Y)]}{\theta = \theta^\star} \Longrightarrow
  \E [X X^\top] \theta^\star = \E[X Y].
\end{align}
We now use this relation to establish the claim.  We have
\begin{align*}
\L(\theta) - \L(\theta^\star) & = \E( \smlinprod{X}{\theta} - Y)^2 -
\E(\smlinprod{X}{\theta^\star} - Y)^2 \\
& = \E\Big[( \smlinprod{X}{\theta} - Y )-(\smlinprod{X}{\theta^\star}
  - Y )\Big] \Big[( \smlinprod{X}{\theta} - Y ) +
  (\smlinprod{X}{\theta^\star} - Y ) \Big] \\
& = \E\Big[(\theta-\theta^\star)^\top X\Big] \Big[X^\top(\theta -
  \theta^\star) - Y + 2 \smlinprod{X}{\theta^\star} - Y \Big] \\
& = \Big[(\theta-\theta^\star)^\top (\E X X ^\top)
   (\theta-\theta^\star) \Big] + 2\E\Big[ (\theta -\theta^\star)^\top
   \left(X X^\top \theta^\star - X Y\right)\Big] \\
& = \| \theta-\theta^\star \|^2_{\E X X^\top} +
 2(\theta-\theta^\star)^\top \underbrace{\Big[ \E(X X
     ^\top)\theta^\star - \E[ X Y] \Big]}_{= 0 \; \text{by
     \cref{eqn:OptimalityConditionUnconstrained}}}.
\end{align*}
\end{proof}

\begin{lemma}[Excess Risk with Regularization]
  \label[lemma]{lem:Sigma2ERM}  
For a fixed $\lambda > 0$, define
\begin{align*}
  \L(w) & = \tfrac{1}{2} \E_{(X,Y)} \big( \smlinprod{X}{w} - Y
  \big)^2, \quad \mbox{and} \quad w^\star \in \argmin_{\| w \|_2 \leq
    \BallRad} \L(w).
\end{align*}
Then for any scalar $M > 0$, we have
\begin{align}
\| w - w^\star\|^2_{(M \E_{X} [X X^\top] + \lambda I)}\leq 2 M \big(
\L(w) - \L(w^\star) \big) + \lambda \| w - w^\star \|_2^2.
\end{align}
\end{lemma}
\begin{proof}
We adopt the shorthand $\E$ for $\E_{(X,Y)}$.  We can write (two
times) the excess risk as
\begin{align}
2\Big[\L(w) - \L(w^\star) \Big] & = \E \left( \smlinprod{X}{w} -
y\right)^2 - \E \left( \smlinprod{X}{w}^\star - y\right)^2 \nonumber \\
& = \E \Big[ \left( \smlinprod{X}{w} - y\right) - \left(
  \smlinprod{X}{w}^\star - Y \right) \Big] \Big[ \smlinprod{X}{w} - Y
  + \smlinprod{X}{w}^\star - Y \Big] \nonumber \\
& = \E \Big[ X^\top \left( w - w^\star \right) \Big] \Big[
  \smlinprod{X}{w} - Y + \smlinprod{X}{w}^\star - y \Big] \nonumber \\
& = \E \Big[ X^\top \left( w - w^\star \right) \Big]\Big[ X^\top (w -
  w^\star) + \smlinprod{X}{w}^\star - y + \smlinprod{X}{w^\star} - Y
  \Big] \nonumber \\
\label{EqnInter}
& = \left( w - w^\star \right)^\top \E \left( X X^\top\right) \left( w
- w^\star \right) + 2\E \Big[ x^\top\left( w - w^\star \right)
  \Big]\Big[ \smlinprod{X}{w^\star} - Y \Big].
\end{align}
The optimality condition for $\wstar$ reads
\begin{align*}
\E \Big[ \big( \smlinprod{X}{\wstar} - Y \big) X^\top \Big] (w -
w^\star) \geq 0 \qquad \mbox{for any feasible $w$.}
\end{align*}
Applying this inequality to equation~\eqref{EqnInter} yields
\begin{align*}
2\Big[\L(w) - \L(w^\star) \Big] & \geq \left( w - w^\star \right)^\top
\E \left( X X^\top\right)\left( w - w^\star \right) \\
& = \left( w - w^\star \right)^\top \Big[ \E X X^\top + \frac{
    \lambda}{M} I \Big]\left( w - w^\star \right) - \frac{ \lambda}{M}
\| w - w^\star \|_2^2,
\end{align*}
as claimed.
\end{proof}




\end{document}